\documentclass[11pt]{article}

\usepackage[final, babel]{microtype}
\usepackage{fullpage}
\usepackage{algorithm,algorithmicx} 
\usepackage[noend]{algpseudocode}
\usepackage{amsmath,amsthm,amsfonts,amssymb}
\usepackage{amsmath}
\usepackage{hyperref}
\usepackage{color}
\usepackage{mathrsfs}
\usepackage{enumitem}
\usepackage{bm}
\usepackage{multirow}
\usepackage{booktabs}
\usepackage{makecell}
\usepackage{graphicx}
\usepackage{subfigure}
\usepackage{caption}
\usepackage{thmtools}
\usepackage{thm-restate}
\usepackage{hhline}
\usepackage[table]{xcolor}
\definecolor{light-gray}{gray}{0.9}

\newcommand{\defeq}{\mathrel{\mathop:}=}

\renewcommand{\epsilon}{\varepsilon}

\newcommand{\vect}[1]{\ensuremath{\mathbf{#1}}}
\newcommand{\mat}[1]{\ensuremath{\mathbf{#1}}}
\newcommand{\dd}{\mathrm{d}}

\newcommand{\argmin}{\mathop{\rm argmin}}
\newcommand{\argmax}{\mathop{\rm argmax}}

\newcommand{\norm}[1]{\left\|{#1}\right\|}
\newcommand{\fnorm}[1]{\left\|{#1}\right\|_{\text{F}}}

\newcommand{\tr}{\mathrm{tr}}
\renewcommand{\det}{\mathrm{det}}

\newcommand{\logdet}{\mathrm{logdet}}
\newcommand{\trans}{^{\top}}

\newcommand{\cO}{\mathcal{O}}

\newcommand{\R}{\mathbb{R}}

\renewcommand{\P}{\mathbb{P}}

\newcommand{\E}{\mathbb{E}}

\newcommand{\cA}{\mathcal{A}}

\newcommand{\cS}{\mathcal{S}}

\newcommand{\A}{\mat{A}}

\newcommand{\I}{\mat{I}}

\newcommand{\U}{\mat{U}}

\newcommand{\e}{\vect{e}}

\renewcommand{\v}{\vect{v}}
\newcommand{\w}{\vect{w}}
\newcommand{\x}{\vect{x}}

\newcommand{\q}{\vect{q}}

\newcommand{\nn}{\nonumber}

\newcommand{\ud}{{\mathrm{d}}}

\usepackage{times}

\newtheorem{theorem}{Theorem}[section]
\newtheorem{lemma}[theorem]{Lemma}

\newtheorem{proposition}[theorem]{Proposition}
\theoremstyle{definition}

\newtheorem{example}[theorem]{Example}
\newtheorem{assumption}{Assumption}

\def\##1\#{\begin{align}#1\end{align}}
\def\$#1\${\begin{align*}#1\end{align*}}

\let\hat\widehat
\let\tilde\widetilde

\def\given{{\,|\,}}

\usepackage[numbers]{natbib}

\begin{document}

\title{\textbf{Provably Efficient Reinforcement Learning with Linear Function Approximation}}

\newcommand{\authorname}[1]{\makebox[6cm][c]{#1}}

\author{Chi Jin \\ University of California, Berkeley \\ \texttt{chijin@cs.berkeley.edu}
\and 
Zhuoran Yang \\ Princeton University \\ \texttt{zy6@princeton.edu}
\and
Zhaoran Wang \\ Northwestern University \\ \texttt{zhaoranwang@gmail.com}
\and
Michael I. Jordan \\ University of California, Berkeley \\ \texttt{jordan@cs.berkeley.edu}}

\date{}

\maketitle

\newcommand{\cnote}{\textcolor[rgb]{1,0,0}{Chi: }\textcolor[rgb]{1,0,1}}
\newcommand{\wnote}{\textcolor[rgb]{1,0,0}{Zhaoran: }\textcolor[rgb]{0,1,1}}
\newcommand{\znote}{\textcolor[rgb]{1,0,0}{Zhuoran: }\textcolor[rgb]{0, 0.5, 1}}

\newcommand{\ind}[3]{{#1}^{#2}_{#3}} 
\newcommand{\sind}[3]{{#1}^{#2}_{#3}} 
\newcommand{\empP}[2]{\hat{\mathbb{P}}^{#1}_{#2}}
\newcommand{\VV}{\mathbb{V}}
\newcommand{\empVV}[2]{\hat{\mathbb{V}}^{#1}_{#2}}
\newcommand{\logt}{\iota}
\newcommand{\cF}{\mathcal{F}}
\newcommand{\fN}{\mathfrak{N}}
\newcommand{\fM}{\mathfrak{M}}
\newcommand{\Ea}{\E_{a^\star}}
\newcommand{\Pa}{\P_{a^\star}}
\newcommand{\done}{\mathbb{I}}
\newcommand{\bphi}{\bm{\phi}}
\newcommand{\bpsi}{\bm{\psi}}
\newcommand{\bmu}{\bm{\mu}}

\newcommand{\bu}{\bm{u}}
\newcommand{\bv}{\bm{v}}
\newcommand{\btheta}{\bm{\theta}}
\newcommand{\bPhi}{\bm{\Phi}}
\newcommand{\bPsi}{\bm{\Psi}}
\newcommand{\fE}{\mathfrak{E}}
\newcommand{\la}{\langle}
\newcommand{\ra}{\rangle}

\begin{abstract}
Modern Reinforcement Learning (RL) is commonly applied to practical problems with an enormous number of states, where \emph{function approximation} must be deployed to approximate either the value function or the policy. The introduction of function approximation raises a fundamental set of challenges involving computational and statistical efficiency, especially given the need to manage the exploration/exploitation tradeoff. As a result, a core RL question remains open: how can we design provably efficient RL algorithms that incorporate function approximation? This question persists even in a basic setting with linear dynamics and linear rewards, for which only linear function approximation is needed. 




This paper presents the first provable RL algorithm with both polynomial runtime and polynomial sample complexity in this linear setting,
without requiring a ``simulator'' or additional assumptions. Concretely, we prove that an optimistic modification of Least-Squares Value Iteration (LSVI)---a classical algorithm frequently studied in the linear setting---achieves $\tilde{\mathcal{O}}(\sqrt{d^3H^3T})$ regret, where $d$ is the ambient dimension of feature space, $H$ is the length of each episode, and $T$ is the total number of steps. Importantly, such regret is independent of the number of states and actions.

\end{abstract}


\section{Introduction}
Reinforcement Learning (RL) is a control-theoretic problem in which an agent tries to maximize
its expected cumulative reward by interacting with an unknown environment over time \cite{sutton2011reinforcement}.
Modern RL commonly engages practical problems with an enormous number of states, where \emph{function approximation}
must be deployed to approximate the \emph{(action-)value function}---the expected cumulative reward
starting from a state-action pair---or the \emph{policy}---the mapping from a state to its subsequent action. Function approximation, especially based on deep neural networks, lies at the heart of the recent practical successes of RL in domains such as Atari games \cite{mnih2013playing}, Go \cite{silver2016mastering}, robotics \cite{kober2012reinforcement}, and dialogue systems \cite{li2016deep}. Moreover, deep neural networks serve as essential components of generic deep RL algorithms, including Deep Q-Network (DQN) \cite{mnih2013playing}, Asynchronous Advantage Actor-Critic (A3C) \cite{mnih2016asynchronous}, and Trust Region Policy Optimization (TRPO) \cite{schulman2015trust}.


Despite the empirical successes of function approximation in RL, 
most existing theoretical guarantees apply only to \emph{tabular} 
RL \cite[see, e.g.,][]{jaksch2010near,osband2014generalization,azar2017minimax,
jin2018q}, in which the states and actions are discrete, and the value 
function is represented by a table. Due to the  
curse of dimensionality, only relatively small problems can be tackled by tabular RL.  Thus, researchers have turned to function 
approximation~\cite[see, e.g.,][]{sutton1988learning, bradtke1996linear, 
tsitsiklis1997analysis}, in theory and in practice.  While function approximation greatly expands the potential
reach of RL, particularly via deep RL architectures, it 
raises a number of fundamental theoretical challenges.  For example,
while the effective state and action spaces can be much larger when function approximation is used, the neighborhoods 
of most states are not visited even once during a set of learning episodes, which makes it difficult to obtain reliable estimates of value functions~\cite[see, e.g.,][]{sutton2011reinforcement, szepesvari2010algorithms, lattimore2018bandit}. To cope with this challenge, relatively simple function classes, including linear function classes, are often used. This introduces, however, a bias, even in the limit of infinite training data, given that the optimal value function and policy may not be linear~\cite[see, e.g.,][]{baird1995residual, boyan1995generalization, tsitsiklis1997analysis}. Thus, both in theory and in practice, the design of RL systems must cope with fundamental statistical problems of sparsity and misspecification, all in the context of a dynamical system.  Moreover, a core distinguishing feature of RL is that it requires addressing the tradeoff between exploration and exploitation.  Addressing this tradeoff algorithmically requires exactly the kinds of statistical estimates that are challenging to obtain in the RL setting due to sparsity, misspecification, and dynamics.  Thus the following fundamental question remains open: 
\begin{center}
\textbf{Is it possible to design provably efficient RL algorithms in the function approximation setting?}
\end{center}
By ``efficient'' we mean efficient in both runtime and sample complexity---the runtime and the sample complexity should not depend on the number of states, but should depend instead on an intrinsic complexity measure of the function class.  

Several recent attempts have been made to attack this fundamental problem. However, they either require the access to a ``simulator'' \cite{yang2019sample} which alleviates the difficulty of exploration, or assume the transition dynamics to be deterministic \cite{wen2013efficient,wen2017efficient}, to have a low variance \cite{du2019provably}, or are parametrizable by a relatively small matrix \cite{yang2019reinforcement}, which alleviates the difficulty in estimating the transition dynamics (see Section \ref{sec:related} for more details).

Focusing on a linear setting in which the transition dynamics and reward function are assumed to be linear, we present the first algorithm that is provably efficient in both runtime and sample complexity, without requiring additional oracles or stronger assumptions. Concretely, in the general setting of an episodic Markov Decision Process (MDP), we prove that an optimistic version of Least-Squares Value Iteration (LSVI) \cite{bradtke1996linear,osband2014generalization}---a classical algorithm frequently studied in the linear setting---achieves $\tilde{\mathcal{O}}(\sqrt{d^3H^3T})$ regret, where $d$ is the ambient dimension of feature space, $H$ is the length of each episode, $T$ is the total number of steps, and $\tilde{\cO}(\cdot)$ hides only absolute constant and poly-logarithmic factors. Importantly, such regret is independent of $S$ and $A$---the number of states and actions. Our algorithm runs in $\cO(d^2 AKT)$ time and $\cO(d^2H + dAT)$ space, which are again independent of $S$ and thus efficient in practice. In addition, our result is robust to the linear assumption: When the underlying transition model is not linear, but $\zeta$-close to linear in total variation distance (Assumption \ref{assumption:nearly_linear}), our algorithm achieves $\tilde{\mathcal{O}}(\sqrt{d^3 H^3 T} + \zeta dHT)$ regret. That is, in addition to the standard $\sqrt{T}$ regret, the algorithm also suffers from a linear regret term that scales with an error $\zeta$ that arises due to the function class misspecification. 








\subsection{Related Work}\label{sec:related}

\paragraph{Tabular RL:}
Tabular RL is well studied in both model-based \cite{jaksch2010near, osband2014generalization, azar2017minimax, dann2017unifying} and model-free settings \cite{strehl2006pac, jin2018q}. See also \cite{koenig1993complexity, azar2011speedy, azar2012sample, lattimore2012pac, sidford2018variance, wainwright2019variance} for a simplified setting with access to a ``simulator'' (also called a generative model), which is a strong oracle that allows the algorithm to query arbitrary state-action pairs and return the reward and the next state. The ``simulator'' significantly alleviates the difficulty of exploration, since a naive exploration strategy which queries all state-action pairs uniformly at random already leads to the most efficient algorithm for finding an optimal policy \cite{azar2012sample}.

In the episodic setting with nonstationary dynamics and no ``simulators,'' the best regrets achieved by existing model-based and model-free algorithms are $\tilde{\mathcal{O}}(\sqrt{H^2 S A T})$ \cite{azar2017minimax} and $\tilde{\mathcal{O}}(\sqrt{H^3 S A T})$ \cite{jin2018q}, respectively, both of which (nearly) attain the minimax lower bound $\Omega(\sqrt{H^2 S A T})$ \cite{jaksch2010near, osband2016lower, jin2018q}.
Here $S$ and $A$ denote the numbers of states and actions, respectively. Although these algorithms are (nearly) minimax-optimal, they can not cope with large state spaces, as their regret scales linearly in $\sqrt{S}$, where $S$ is often exponentially large in practice \cite[see, e.g.,][]{mnih2013playing, silver2016mastering, kober2012reinforcement, li2016deep}. Moreover, the minimax lower bound suggests that, information-theoretically, a large state space cannot be handled efficiently unless further problem-specific structure is exploited. Compared with this line of work, in the current paper we exploit the linear structure of the reward and transition functions and show that the regret of optimistic LSVI scales polynomially in the ambient dimension $d$ rather than the number of states $S$.

\paragraph{Linear bandits:} To enable function approximation, another line of related work studies stochastic linear bandits or stochastic linear contextual bandits \cite[see, e.g.,][]{auer2002using, dani2008stochastic, li2010contextual, rusmevichientong2010linearly, chu2011contextual, abbasi2011improved}, which is a special case of the linear MDP studied in this paper (Assumption \ref{assumption:linear}) with the episode length $H$ set equal to one. See \cite{bubeck2012regret, lattimore2018bandit} and the references therein for a detailed survey. The best regrets achieved by existing algorithms are $\tilde{\mathcal{O}}(d\sqrt{T})$ for linear bandits \cite{abbasi2011improved} and $\tilde{\mathcal{O}}(\sqrt{d T})$ for linear contextual bandits \cite{auer2002using,chu2011contextual}, both of which scale polynomially in the ambient dimension $d$. We note, however, that while an MDP has state transition, linear bandits do not. This temporal structure captures the fundamental difference in their difficulties of exploration: a naive adaptation of existing linear bandit algorithms to the linear MDP setting yields a regret exponential in $H$---the length of each episode.




\paragraph{RL with function approximation:} In the setting of linear function approximation, there is a long line of classical work on the design of algorithms, but this work does not provide polynomial sample efficiency guarantees~\cite[see, e.g.,][]{bradtke1996linear, melo2007q, sutton2011reinforcement, osband2014generalization, azizzadenesheli2018efficient}.  Recently, Yang and Wang \cite{yang2019sample} revisited the setting of linear transitions and rewards \cite{bradtke1996linear, melo2007q} (Assumption \ref{assumption:linear}), and presented a sample-efficient algorithm assuming the access to a ``simulator''. Similar to the case of tabular setting, the ``simulator'' greatly alleviates the difficulty of exploration. We also note that their very recent work \cite{yang2019reinforcement}, developed independently of the current paper, provides sample efficiency guarantees for exploration in the linear MDP setting. Compared with the current paper, \cite{yang2019reinforcement} differs in that requires one additional key assumption---that the transition model can be parameterized by a relatively small matrix. This additional assumption reduces the number of free parameters in the transition model from potentially being infinite (for the case with an infinite number of states) 
to small and finite, and thus mitigates the challenges in estimating the transition model. As a result, their algorithm and main mechanism are based on estimating the unknown matrix, which differs from our approach. Finally, in a broader context, without the assumption of a linear MDP, sample efficiency guarantees have been established for RL under other assumptions, such as that the transition dynamics are fully deterministic \cite{wen2013efficient, wen2017efficient}, or have low variances \cite{du2019provably}. These assumptions can be potentially restrictive in practice, and may not hold even in the tabular setting. In contrast, our results directly cover the standard tabular case with no extra assumptions.

In the setting of general function approximation, Jiang et al. \cite{jiang2017contextual} present a generic algorithm Olive, which enjoys sample efficiency if a complexity measure that they refer to as ``Bellman rank'' is small. It can be shown that Bellman rank is at most $d$ under Assumption \ref{assumption:linear}, and thus Olive is sample efficient in our setting. In contrast to our results, Olive is not computationally efficient in general and it does not provide a $\sqrt{T}$ regret bound. Meanwhile, a recent line of work \cite{zhu2019lipschitz, wang2019towards} studies a nonparametric setting with H\"older smooth reward and transition model. The sample complexities provided therein are exponential in dimensionality in the worst case.

\section{Preliminaries}

We consider the setting of an episodic Markov decision process, denoted by $\rm{MDP}(\cS, \cA, H, \P, r)$, where 
$\cS$ and $\cA$ are the sets of possible states and actions, respectively, 
$H \in \mathbb{Z_{+}}$ is the length of each episode,
$\P = \{ \P_h \}_{h=1}^H $ and $ r = \{ r_h \}_{h=1}^H $  are the  state transition probability measures and the reward functions,  respectively.
We assume that   $\cS$ is a  \emph{measurable space}  with possibly infinite number of elements and $\cA $ is a finite set with cardinality $A $. Moreover, for each $h \in [ H]$, $\P_h ( \cdot | x, a) $ denotes the transition kernel 
over the next states if action $a$ is taken for state $x$ at step $h\in [H]$, and $r_h \colon \cS \times \cA \to [0,1]$ is the deterministic reward function at step $h$.%
\footnote{While we study deterministic reward functions for notational simplicity, our results readily generalize to random  reward functions. Also, we assume the reward lies in $[0,1]$ without loss of generality.}

An agent interacts with this episodic MDP as follows. In each episode, an initial state $x_1$ is   picked arbitrarily by an adversary. 
Then, at each step $h \in [H]$, the agent observes the state $x_h \in \cS$, picks an action $a_h \in \cA$, and receives a reward $r_h(x_h, a_h)$. Moreover, the MDP evolves into a new state  $x_{h+1}$  that 
is drawn from the probability measure $\P_h(\cdot | x_h, a_h)$. The episode terminates when $x_{H+1}$ is reached. We note that the agent cannot take an action at $x_{H+1}$ and hence receives no reward.

%

A policy $\pi$ of an agent is a function $\pi: \cS \times [H] \rightarrow \cA$, where  $\pi(x, h)$ is the action that the agent takes at state $x$ and at the $h$th step in the episode. 
Moreover, for each $h\in [H]$, we define the value function $V_h^{\pi}\colon \cS \to \mathbb{R} $ as the expected value of cumulative  rewards received under policy $\pi$ when starting from an arbitrary state at the $h$th step. Specifically, we have 
\begin{equation*}
 \sind{V}{\pi}{h}(x) \defeq \E\left[\sum_{h' = h}^H r_{h'}(x_{h'}, \pi(x_{h'}, h'))  \bigg | x_h = x\right], \qquad \forall x\in \cS, h \in [H].
\end{equation*}
Accordingly, we also define the action-value function  $\ind{Q}{\pi}{h}:\cS \times \cA \to \mathbb{R}$ which gives the expected value of cumulative  rewards when the agent starts  from an  arbitrary state-action pair at the $h$-th step and follows  policy $\pi$ afterwards; that is, 
\begin{equation*}
 \sind{Q}{\pi}{h}(x,a) \defeq r_h(x, a) + \E  \bigg[\sum_{h' = h+1}^H r_{h'}(x_{h'}, \pi(x_{h'}, h')) \bigg  | x_h = x, a_h = a \bigg], \qquad \forall (x,a) \in \cS\times \cA, \forall h \in [H] .
\end{equation*}
Since the action spaces and the episode length are both finite, there always exists 
 an optimal policy $\pi^\star$ which gives the optimal value $\sind{V}{\star}{h}(x) = \sup_{\pi} V_h^\pi(x)$ for all $x\in \cS$ and $h\in [H]$ \cite[see, e.g.,][]{puterman2014markov}).
To simplify the notation,  we denote  $[\P_h\sind{V}{}{h+1}](x, a) \defeq \E_{x' \sim \P_h(\cdot|x, a)}\ind{V}{}{h+1}(x')$. Using this notation, the   Bellman equation associated with a policy $\pi$ becomes 
\begin{align} \label{eq:bellman} 
    \sind{Q}{\pi}{h}(x, a) = (r_h + \P_h \sind{V}{\pi}{h+1})(x, a) , \qquad   \sind{V}{\pi}{h}(x) = \sind{Q}{\pi}{h}(x, \pi_h(x)),    
    \qquad \sind{V}{\pi}{H+1}(x) = 0,  
\end{align}
which holds for all  $(x,a) \in \cS \times \cA $. Similarly, the Bellman optimality equation is 
\begin{align}\label{eq:opt_bellman}
  \sind{Q}{\star}{h}(x, a) = (r_h + \P_h \sind{V}{\star}{h+1})(x, a) , \qquad 
 \sind{V}{\star}{h}(x) = \max_{a\in\cA}\sind{Q}{\star}{h}(x, a), \qquad \sind{V}{\star}{H+1}(x) = 0. 
\end{align}
This implies that the optimal policy $\pi^\star$ is the greedy policy with respect to the optimal action-value function $\{ Q^\star _h \}_{h \in [H]}$. Thus, to find the optimal policy $\pi^\star$, it suffices to estimate  the optimal action-value functions. 

Furthermore, under the setting of an episodic MDP, the agent aims to learn the optimal policy by interacting with the environment during a set of episodes. For each $k \geq 1$, at the beginning of the $k$th episode, the 
adversary picks the initial  state $x^k_1$ and  the agent chooses  policy $\pi_k$. 
The difference in values between $V^{\pi_k}_1 (x_1^k)$ and $\sind{V}{\star}{1} (\ind{x}{k}{1}) $ serves as the expected regret or the suboptimality of  the agent at the $k$-th episode. 
Thus, after playing for $K$ episodes, 
the total (expected) regret is 
\begin{equation*}
\text{Regret}(K) = \sum_{k=1}^K \left[\sind{V}{\star}{1} (\ind{x}{k}{1}) - \sind{V}{\pi_k}{1} (\ind{x}{k}{1})\right].
\end{equation*}

\subsection{Linear Markov decision processes}

We focus on a setting of a \emph{linear Markov decision process}, where the transition kernels and the reward function are assumed to be linear.  This assumption implies that the action-value function is linear, as we will show.  Note that this is \emph{not} the same as the assumption that the policy is a linear function---an assumption that has been the focus of much of the literature. Rather, it is akin to a statistical modeling assumption, in which we make assumptions about how data are generated and then study various estimators.  Formally, we make the following definition.


\begin{assumption}[Linear MDP \cite{bradtke1996linear, melo2007q}] \label{assumption:linear}
$\rm{MDP}(\cS, \cA, H, \P, r)$ is a \emph{linear MDP} with a feature map $\bphi: \cS \times \cA \rightarrow \R^d$, if for any $h\in [H]$, there exist $d$ \emph{unknown} (signed) measures $\bmu_h = (\mu_h^{(1)}, \ldots, \mu_h^{(d)})$ over $\cS$ and an \emph{unknown} vector $\btheta_h \in \R^d$, such that 
 for any $(x, a) \in \cS \times \cA$, we have 
\begin{align}\label{eq:linear_transition}  
\P_h(\cdot\given x, a) = \la\bphi(x, a), \bmu_h(\cdot)\ra, \qquad 
r_h(x, a) = \la\bphi(x, a), \btheta_h\ra.  
\end{align}
Without loss of generality, we assume $\norm{\bphi(x, a)} \le 1$ for all $(x,a ) \in \cS \times \cA$, and $\max\{\norm{\bmu_h(\cS)}, \norm{\btheta_h}\} \le \sqrt{d}$ for all $h \in [H]$.
\end{assumption}
By definition, in a linear MDP, both the Markov transition model and the reward functions are linear in a feature mapping $\bphi$. We remark that despite being linear, the Markov transition model $\P_h(\cdot|x, a)$ can still have infinite degrees of freedom as the measure $\bmu_h$ is unknown. This is a key difference from the linear quadratic regulator \citep{abbasi2011regret, dean2018regret, abeille2018improved, abbasi2019model, cohen2019learning} or the recent work of Yang and Wang \citep{yang2019reinforcement}, whose transition models are completely specified by a finite-dimensional matrix such that the degrees of freedom are bounded.  

Recall that we assume the reward functions are bounded in $[0,1]$, which implies that the value functions are bounded in $[0, H]$. Our choice of normalization conditions in Assumption \ref{assumption:linear} implies that the following concrete examples serve as special cases of a linear MDP.

\begin{example}[Tabular MDP] \label{ex:tabular}For the scenario with finitely many states and actions, letting $d =  | \cS | \times | \cA|$, then each coordinate can be indexed by state-action pair $(x, a) \in \cS \times \cA$. Let $\bphi (x, a) = \e_{(x, a)}$ be the canonical basis in $\R^d$. Then if we set $\e_{(x, a)}\trans\bmu_h(\cdot) = \P_h(\cdot|x, a)$ and $\e_{(x, a)}\trans\btheta_h = r_h(x, a)$ for any $h \in [H]$, we recover the tabular MDP.
\end{example}

\begin{example}[Simplex Feature Space] When the feature space, $\{ \bphi (x, a) \colon (x,a) \in \cS \times \cA\}$, is a subset of the $d$-dimensional simplex, $\{\bpsi| \sum_{i=1}^d \psi_i = 1 \text{~and~} \psi_i \ge 0  \text{~for all~}i \}$,  a linear MDP can be instantiated by choosing $\e_i\trans\bmu_h$ to be an arbitrary probability measure over $\cS$ and letting $\btheta_h$ be any vector such that $\norm{\btheta_h}_\infty \le 1$.
\end{example}

As mentioned earlier, a crucial property of the linear MDP is that, for all policies, the action-value functions are always linear in the feature map $\bphi$. Therefore, when designing RL algorithms, it suffices to focus on linear action-value functions.






\begin{restatable}{proposition}{PROPrealizable} \label{prop:realizable}
For a linear MDP, for any policy $\pi$, there exist weights $\{\w^{\pi}_h\}_{h\in[H]}$ such that for any $(x, a, h) \in \cS\times\cA\times[H]$, we have $Q_h^\pi(x, a) = \la\bphi (x, a),  \w^{\pi}_h\ra$.
\end{restatable}

We provide a proof of this proposition in Appendix \ref{sec:properties}, where we also present additional  discussion of the basic properties of a linear MDP.


\section{Main Results} \label{sec:results}

In this section, we present our main results, which provide sample complexity guarantees for Algorithm \ref{algo:LSVI-UCB} in the linear MDP setting (Theorem \ref{thm:main}) and in a misspecified setting (Theorem \ref{thm:main_mis}). 

\begin{algorithm}[t]
\caption{Least-Squares Value Iteration with UCB (LSVI-UCB)}\label{algo:LSVI-UCB}
\begin{algorithmic}[1]
\For{episode $k = 1, \ldots, K$}
\State Receive the initial state $x^k_1$.
\For{step $h = H, \ldots, 1$}
\State $\Lambda_h \leftarrow \sum_{\tau =1}^{k-1}  \bphi(x^{\tau}_h,  a^{\tau}_h)\bphi(x^{\tau}_h, a^{\tau}_h)\trans + \lambda \cdot  \I$. \label{line:Lambda}
\State $\w_h \leftarrow \Lambda_h^{-1} \sum_{\tau=1}^{k-1} \bphi(x^{\tau}_h,  a^{\tau}_h) [r_h(x^{\tau}_h, a^{\tau}_h) + \max_a Q_{h+1}(x^{\tau}_{h+1}, a)]$. \label{line:w}
\State $Q_h(\cdot, \cdot) \leftarrow \min\{\w_h\trans\bphi(\cdot, \cdot) + \beta  [\bphi(\cdot, \cdot)^\top \Lambda_h^{-1} \bphi(\cdot, \cdot)]^{1/2}, H\}$. \label{line:ucb}
\EndFor
\For{step $h = 1, \ldots, H$}
\State Take action $a^k_h \gets  \argmax_{a \in \cA } Q_h(x^k_h, a)$, and observe $x^k_{h+1}$. 
\EndFor
\EndFor
\end{algorithmic}
\end{algorithm}

We first lay out our algorithm (Algorithm \ref{algo:LSVI-UCB})---an optimistic modification of Least-Square Value Iteration (LSVI), where the optimism is realized by Upper-Confidence Bounds (UCB). At a high level, each episode consists of two passes (or loops) over all steps. The first pass (line 3-6) updates the parameters $(\w_h, \Lambda_h)$ that are used to form the action-value function $Q_h$. The second pass (line 7-8) executes the greedy policy, $a_h = \argmax_{a \in \cA} Q_h(x_h, a)$, according to the $Q_h$ obtained in the first pass. We note $Q_{H+1}(\cdot, \cdot)\equiv 0$ since the agent receives no reward after the $H$th step. For the first episode $k=1$, since the summation in line 4-5 is from $\tau = 1$ to $0$, we simply have $\Lambda_h \leftarrow \lambda \I$ and $\w_h \leftarrow 0$. Line 6 specifies the dependency of the action-value function $Q_h$ on the parameters $\w_h$ and $\Lambda_h$, and no actual updates need to be performed.

The idea of Least-Square Value Iteration \cite{bradtke1996linear,osband2014generalization} stems from the classical value-iteration algorithm, which finds the optimal policy (or action-value function) by applying the Bellman optimality equation Eq.~\eqref{eq:opt_bellman} recursively:
\begin{equation*}
\sind{Q}{\star}{h}(x, a) \leftarrow \bigl[r_h + \P_h \max_{a'\in\cA}\sind{Q}{\star}{h+1}(\cdot, a')\bigr](x, a), \quad \forall(x, a) \in \cS\times \cA.
\end{equation*}
In practical RL with linear function approximation, there are two challenges to face in implementing the updates: First, $\P_h$ is unknown, and it is replaced by the samples observed empirically. Second, in the setting of large state space, we cannot iterate over all $(x, a)$. We parametrize $Q_h^\star(x, a)$ by a linear form $\w_h\trans \bphi(x, a)$ instead. A natural idea here is to replace the Bellman update by solving for $\w_h$ in a least-squares problem. In fact, the update of $\w_h$ in Algorithm \ref{algo:LSVI-UCB} solves precisely the following regularized least-squares problem:
\begin{equation*}
\w_h \leftarrow \argmin_{\w \in \R^d} \sum_{\tau = 1}^{k-1} \bigl [r_h ( x_h^\tau , a_h^\tau) + \max_{a \in \cA } Q_{h+1} ( x_{h+1} ^\tau, a) - \w ^\top \bphi ( x_h^{\tau },  a_h^\tau )    \bigr ]^2 + \lambda  \| \w \|^2.
\end{equation*}
Algorithm \ref{algo:LSVI-UCB} additionally adds an UCB bonus term of form $\beta (\bphi\trans \Lambda_h^{-1} \bphi)^{1/2}$ to encourage exploration, where $\Lambda_h$ is the Gram matrix of the regularized least-squares problem, and $\beta$ is a scalar. This form of bonus is common in the literature on linear bandits \cite{bubeck2012regret, lattimore2018bandit}. Intuitively, $m \defeq (\bphi\trans \Lambda_h^{-1} \bphi)^{-1}$ represents the effective number of samples the agent has observed so far along the $\bphi$ direction, and thus the bonus term $\beta/\sqrt{m}$ represents the uncertainty along the $\bphi$ direction. It is called an upper confidence bound because, by choosing a proper value for $\beta$ we can prove that, with high probability, $Q_h$ in line \ref{line:w} of Algorithm \ref{algo:LSVI-UCB} is always an upper bound of $Q_h^\star$ for all state-action pair (see Lemma \ref{lem:UCB}).

We are now ready to state our main theorem, which gives a $\sqrt{T}$-regret bound in the linear MDP setting without any further assumptions. Here, $T=KH$ is the total number of steps.

\begin{restatable}{theorem}{THMmain} \label{thm:main}
Under Assumption \ref{assumption:linear}, there 
exists an absolute constant $c >0$ such that, for any fixed $p \in (0, 1)$, if we set $\lambda = 1$ and  $\beta =  c \cdot d H \sqrt{\logt} $ in Algorithm \ref{algo:LSVI-UCB} with $\logt \defeq \log (2dT/p)$,   
 then with probability $1-p$, the total regret of LSVI-UCB (Algorithm \ref{algo:LSVI-UCB}) is at most $\cO(\sqrt{d^3 H^3 T\logt^2})$, where $\cO(\cdot)$ hides only absolute constants. 
\end{restatable}

Theorem \ref{thm:main} asserts that when $\lambda$ and $\beta$ are set properly, LSVI-UCB will suffer total regret at most $\tilde{\cO}(\sqrt{ d^3 H ^3 T })$. We emphasize that while a naive adaptation of existing linear bandit algorithms to this linear MDP setting easily yields a regret exponential in $H$, our regret is only polynomial in $H$. Avoiding this exponential dependency on the planning horizon is a key step in efficiently solving the sequential RL problem. Additionally, comparing to the minimax regret in a tabular setting, $\tilde{\Theta}(\sqrt{H^2SAT})$, our regret replaces the number of state-action pairs $SA$ by a polynomial dependency on the intrinsic complexity measure of feature space, $d$. In fact, our regret is completely independent of $S$ and $A$, which is crucial in the large state-space setting where function approximation is necessary. Please see also Section \ref{sec:discussion} for more discussion on the optimal dependencies on $d$ and $H$.


We remark that Algorithm \ref{algo:LSVI-UCB} only needs to store $\Lambda_h, \w_h$, $r(x_h^k, a_h^k)$ and $\{\bphi(x_{h}^k, a)\}_{a\in\cA}$ for all $(h, k)\in [H]\times[K]$, which takes $\cO(d^2H + dAT)$ space. When we compute $\Lambda_h^{-1}$ by the Sherman-Morrison formula, the computational complexity of Algorithm \ref{algo:LSVI-UCB} is dominated by line 5 in computing $\max_a Q_{h+1}(x^{\tau}_{h+1}, a)$ for all $\tau \in [k]$. This takes $\cO(d^2 A K)$ time per step, which gives a total runtime $\cO(d^2 AKT)$.


Finally, similarly to the discussion in Section 3.1 of \cite{jin2018q}, our regret bound (Theorem \ref{thm:main}) directly translates to a sample complexity guarantee (or a PAC guarantee) in the following sense. When the initial state $x_1$ is fixed for all episodes, then, with at least constant probability, we can learn an $\epsilon$-optimal policy $\pi$ which satisfies $V^\star (x_1) - V^\pi (x_1 ) \leq \epsilon$ using $\tilde{\cO}(d^3H^4/\epsilon^2)$ samples. The algorithm to achieve this is to simply run Algorithm \ref{algo:LSVI-UCB} for $K=\tilde{\cO}(d^3H^3/\epsilon^2)$ episodes, and then output the greedy policy according to the action-value function $Q$ at the $k$th episode, where $k$ is sampled uniformly from $[K]$.

\subsection{Results for a misspecified setting}

Theorem \ref{thm:main} hinges on the fact that the MDP has a linear structure. A natural follow-up question arises: what would happen if the underlying MDP is not linear, and thus misspecified? We first present a definition for an approximate linear model.


\begin{assumption}[$\zeta$-Approximate Linear MDP]\label{assumption:nearly_linear} For any $\zeta\le 1$, we say that
$\rm{MDP}(\cS, \cA, H, \P, r)$ is a \emph{$\zeta$-approximate linear MDP} with a feature map $\bphi: \cS \times \cA \rightarrow \R^d$, if for any $h\in [H]$, there exist $d$ \emph{unknown} (signed) measures $\bmu_h = (\mu_h^{(1)}, \ldots, \mu_h^{(d)})$ over $\cS$ and an \emph{unknown} vector $\btheta_h \in \R^d$ such that for any $(x, a) \in \cS\times \cA$, we have 
\begin{align}\label{eq:nearly_linear_transition} 
\|\P_h(\cdot\given x, a) -  \la\bphi(x, a), \bmu_h(\cdot)\ra \|_{\mathrm{TV}} \le \zeta, \qquad  
|r_h(x, a) - \la\bphi(x, a), \btheta_h\ra| \le \zeta.   
\end{align}
Without loss of generality, we assume that $\norm{\bphi(x, a)} \le 1$ for all $(x,a ) \in \cS \times \cA$, and $\max\{\norm{\bmu_h(\cS)}, \norm{\btheta_h}\} \le \sqrt{d}$ for all $h \in [H]$.
\end{assumption}

By definition, an MDP is an $\zeta$-approximately linear MDP if there exists a linear MDP such that their Markov transition dynamics and reward functions are close. Here the closeness between transition dynamics is measured in terms of total variation distance. 
 
In general, an algorithm designed for a linear MDP could break down entirely if the underlying MDP is not linear. The following theorem states that this is not the case for our algorithm. It is in fact robust to small model misspecification. To achieve this, we need only to adopt a different hyperparameter $\beta$ in different episodes.






\begin{restatable}{theorem}{THMmainmis} \label{thm:main_mis}
Under Assumption \ref{assumption:nearly_linear}, there 
exists an absolute constant $c >0$ such that,  for any fixed $p \in (0, 1)$, if we set $\lambda = 1$ and  $\beta_k  =  c \cdot (d\sqrt{\logt} + \zeta\sqrt{kd})H $ in Algorithm \ref{algo:LSVI-UCB} with $\logt \defeq \log (2dT/p)$,   
then with probability $1-p$, the total regret of LSVI-UCB (Algorithm \ref{algo:LSVI-UCB}) is at most $\cO \bigl ( \sqrt{d^3 H^3 T\logt^2} + \zeta dHT\sqrt{\logt} \bigr )$. 
\end{restatable}

Compared with Theorem \ref{thm:main}, Theorem \ref{thm:main_mis} asserts that the LSVI-UCB algorithm will incur at most an additional $\tilde{\cO}(\zeta dHT)$ regret when the model is misspecified. This additional term is inevitably linear in $T$ due the intrinsic bias introduced by linear approximation. When $\zeta$ is sufficiently small, i.e., the underlying MDP is not far away from being linear, our algorithm will still enjoy good theoretical guarantees.

Theorem \ref{thm:main_mis} can also be converted to a PAC guarantee with a similar flavor. 
When the initial state $x_1$ is fixed for all episodes, then, with at least constant probability, we can learn an $\epsilon$-optimal policy $\pi$ which satisfies $V^\star (x_1) - V^\pi (x_1 ) \leq \epsilon + \tilde{\cO}(\zeta d H^2)$ using $\tilde{\cO}(d^3H^4/\epsilon^2)$ samples.



\section{Mechanisms}

In this section, we overview several of the key ideas behind the regret bound in Theorem \ref{thm:main}. We defer the full proof of Theorem \ref{thm:main} and Theorem \ref{thm:main_mis} to Appendix \ref{app:proof_main} and Appendix \ref{app:proof_main_mis} respectively.

In Section \ref{sec:results}, we mentioned that the LSVI algorithm is motivated from the Bellman optimality equation Eq.~\eqref{eq:opt_bellman}. It remains to verify that line \ref{line:w} in Algorithm \ref{algo:LSVI-UCB} indeed well approximates the Bellman optimality equation, 
which turns out to require not only the linear MDP structure but also hinges on several other facts. 

To simplify our presentation, in this section we treat the regularization parameter $\lambda$ loosely as being sufficiently small so that $\Lambda_h^{-1} \sum_{\tau =1}^{k-1}  \bphi(x^{\tau}_h,  a^{\tau}_h)\bphi(x^{\tau}_h, a^{\tau}_h)\trans \approx \I$. We will focus in this section on a fixed episode $k$, and drop the dependency of parameters and value functions on $k$ when it is clear from the context. Now, ignoring the UCB bonus, the least-squares solution (line \ref{line:w}) gives the following estimate of the action-value function:
\begin{equation*}
Q_h(x, a) \approx \bphi(x, a)\trans \w_h = \bphi(x, a)\trans\Lambda_h^{-1} \sum_{\tau=1}^{k-1} \bphi(x^{\tau}_h,  a^{\tau}_h) [r_h(x^{\tau}_h, a^{\tau}_h) +  V_{h+1}(x^{\tau}_{h+1})],
\end{equation*}
where $V_{h+1}(\cdot) = \max_{a\in\cA} Q_{h+1}(\cdot, a)$. Plugging in $r_h(\cdot, \cdot) = \bphi(\cdot, \cdot)\trans \btheta_h$, we know the first term on the right-hand side approximates $r_h(x, a)$. Comparing this to Eq.~\eqref{eq:opt_bellman}, it remains to show why the second term of right-hand side approximates $\P_h V_{h+1}(x, a)$. We thus define our empirical Markov transition measure as
\begin{equation*}
\hat{\P}_h (\cdot | x, a) \defeq \bphi(x, a)\trans\Lambda_h^{-1} \sum_{\tau=1}^{k-1} \bphi(x^{\tau}_h,  a^{\tau}_h) \delta(\cdot, x^{\tau}_{h+1}),
\end{equation*}
where the $\delta$-measure $\delta(\cdot, x)$ puts an atom on element $x$. It remains to verify that $\hat{\P}_h V_{h+1}(x, a) \approx \P_h V_{h+1}(x, a)$.
To establish this, we use a measure $\bar{\P}_h$ to bridge these two quantities: 
\begin{equation} \label{eq:P_bar}
\bar{\P}_h (\cdot | x, a) \defeq \bphi(x, a)\trans\Lambda_h^{-1} \sum_{\tau=1}^{k-1} \bphi(x^{\tau}_h,  a^{\tau}_h) \P_h(\cdot | x^{\tau}_h,  a^{\tau}_h).
\end{equation}
Our analysis depends on the following two key steps.

\paragraph{Step 1: Prove $\hat{\P}_h V_{h+1}(x, a) \approx \bar{\P}_h V_{h+1}(x, a)$ via  Value-Aware Uniform Concentration.} Computing the difference, we have $(\hat{\P}_h - \bar{\P}_h)V_{h+1}(x, a) 
= \bphi(x, a)\trans\Lambda_h^{-1} \sum_{\tau=1}^{k-1} \bphi(x^{\tau}_h,  a^{\tau}_h) [V_{h+1}(x^\tau_{h+1}) - \P_h V_{h+1}(x^{\tau}_h,  a^{\tau}_h)]$. Since $\x^\tau_{h+1}$ is a sample from the distribution $\P_h(\cdot|x^{\tau}_h,  a^{\tau}_h)$, we would expect this term to be small due to concentration. This would be the case if function $V_{h+1}$ is \emph{fixed and independent of the samples} $\{x^\tau_{h+1}\}_{\tau=1}^{k-1}$. Then, $V_{h+1}(x^\tau_{h+1}) - \P_h V_{h+1}(x^{\tau}_h,  a^{\tau}_h)$ is a zero-mean random variable in $[-H, H]$, and we could aim to use a concentration inequality for self-normalized processes to bound $(\hat{\P}_h - \bar{\P}_h)V_{h+1}(x, a)$. Please see Theorem \ref{thm:self_norm} or \cite{abbasi2011improved} for more detail on this approach.

However, the function $V_{h+1}$ in Algorithm \ref{algo:LSVI-UCB} is again computed by least-squares value iteration in later steps $[h+1, H]$ and it thus inevitably depends on the choices of actions $\{a^\tau_{h+1}\}_{\tau=1}^{k-1}$, and thus also samples $\{x^\tau_{h+1}\}_{\tau=1}^{k-1}$. Therefore, the concentration of self-normalized process does not apply directly. To resolve this issue, we establish the uniform concentration over all value functions in the following class:
\begin{equation}\label{eq:class_value}
\mathcal{V} = \Big \{ V(\cdot ) | V(\cdot ) = \min \bigl  \{ \max_{a \in \cA} \bphi (\cdot, a)  ^\top  \w + \beta \sqrt{\bphi (\cdot, a) \Lambda^{-1} \bphi (\cdot, a)}, H\bigr \}, \w \in \R^d, \beta \in \R, \Lambda \in \R^{d\times d} \Bigr \},
\end{equation}
where the parameters $\w, \beta, \Lambda$ are all bounded. We ensure that Algorithm \ref{algo:LSVI-UCB} only uses value functions within this class $\mathcal{V}$, which has a reasonably small covering number. This gives, with high probability, $|(\hat{\P}_h - \bar{\P}_h)V_{h+1}(x, a)| \le \tilde{\cO}(dH)\cdot(\bphi(x, a) \Lambda^{-1}_h \bphi(x, a))^{1/2}$  (Lemma \ref{lem:stochastic_term}).

\paragraph{Step 2: Show $\bar{\P}_h V_{h+1}(x, a) \approx \P_h V_{h+1}(x, a)$ due to Linear Markov Transitions.} One big challenge in RL with function approximation is that, due to the large state space, the learner may never visit the neighborhood of a state-action pair twice. This raises a question of how to use the experiences from other state-action pairs to infer information about a state-action pair of interest. In Eq.~\eqref{eq:P_bar}, $\bar{P}_h(\cdot|x, a)$ provides such an estimate via regularized least-squares. Our modeling assumption of a linear MDP (Assumption \ref{assumption:linear}) ensures that this least-square estimate is valid: since $\P_h(\cdot|x, a)=\bphi(x, a)\trans\bmu_h(\cdot)$ for any $(x, a)$ pair, we have
\begin{equation*}
\bar{\P}_h (\cdot | x, a) = \bphi(x, a)\trans\Lambda_h^{-1} \sum_{\tau=1}^{k-1} \bphi(x^{\tau}_h,  a^{\tau}_h)\bphi(x^{\tau}_h,  a^{\tau}_h)\trans\bmu_h(\cdot)
\approx \bphi(x, a)\trans\bmu_h(\cdot) = \P_h (\cdot | x, a).
\end{equation*}

In summary, combining step 1 and step 2, we establish $\hat{\P}_h V_{h+1}(x, a) \approx \P_h V_{h+1}(x, a)$, and hence show that LSVI approximates the optimal Bellman equation. We emphasize that despite being linear, the Markov transition model $\P_h(\cdot|x, a) = \bphi(x, a)\trans \bmu_h(\cdot)$ can still have infinite degrees of freedom since the measure $\bmu_h$ is unknown. Therefore, within a finite number of samples, no algorithm can establish that $\hat{\P}_h$ and $\P_h$ are close in total variation distance. In contrast, our algorithm only requires $\hat{\P}_h V_{h+1}(x, a) \approx \P_h V_{h+1}(x, a)$ for all value functions $V_{h+1}$ in a small function class $\mathcal{V}$ (especially in step 1). This bypasses the need for fully learning the transition model $\P_h$. Thus, our algorithm can also be viewed as ``\emph{model-free}'' in this sense.  

Finally, with the above key observations in mind, our proof proceeds by leveraging and adapting techniques from the literature on tabular MDP and linear bandits. Please see Appendix \ref{app:proof_main} and \ref{app:proof_main_mis} for the details.

\section{Conclusion} \label{sec:discussion}

In this paper, we have presented the first provable RL algorithm with both polynomial runtime and polynomial sample complexity for linear MDPs, without requiring a ``simulator'' or additional assumptions. The algorithm is simply Least-Squares Value Iteration---a classical RL algorithm
commonly studied in the setting of linear function approximation---with a UCB bonus. We hope that our work may serve as a first step towards a better understanding of efficient RL with function approximation.

We provide a few additional concluding observations.

\paragraph{On the optimal dependencies on $d$ and $H$.} Theorem \ref{thm:main} claims the total regret to be upper bounded by $\tilde{\mathcal{O}}(\sqrt{d^3H^3T})$. One immediate question is what the optimal dependencies on $d$ and $H$ are. Since our setting covers the standard tabular setting, as in shown in Example \ref{ex:tabular}, a lower bound can be directly obtained through a reduction from the tabular setting, which gives $\Omega(\sqrt{dH^2T})$ for the case of nonstationary transitions \cite{jin2018q}. We believe the $\sqrt{H}$ difference between this lower bound and our upper bound is expected because the exploration bonus used in this paper is intrinsically ``Hoeffding-type.'' Using a ``Bernstein-type'' bonus can potentially help shave off one $\sqrt{H}$ factor (see \cite{azar2017minimax,jin2018q} for a similar phenomenon in the tabular setting).

In contrast, the optimal dependency on dimension $d$ is more important but is also less clear. In the case where the number of actions is very large, one may attempt to use the lower bound in the linear bandit setting, $\Omega(d\sqrt{T})$, for the case $H=1$. We comment that as soon as $H\ge 2$ (where the Markov transition matters), the assumption of a linear MDP imposes structure on the feature space $\{\bphi(x, a)|(x, a) \in \cS \times\cA\}$ (see Proposition \ref{prop:linear_structure}). Technically, the standard constructions for the hard instances in the linear bandit lower bound do not respect this structure, so the lower bound does not directly apply.  
It remains an interesting future direction to determine this optimal dependency on $d$.

\paragraph{On the assumption of linear transition dynamics.} The main assumption in this paper is the linear MDP assumption (Assumption \ref{assumption:linear}), which requires the Markov transition $\P_h(\cdot|x, a)$ to be linear in $\bphi(x, a)$. This requirement could be strong in practice. It turns out that our proof only relies on a weaker version of this assumption:
\begin{equation} \label{eq:weaker_assumption}
\P_h V(x, a) = \la \bphi(x, a), \w_{V}\ra, \text{~for all~} V\in\mathcal{V}, 
\end{equation}
where $\w_{V}$ is a vector independent of $(x, a)$ and $\mathcal{V}$ is the class of value functions considered in this paper, as in Eq.~\eqref{eq:class_value}. That is, we effectively only need that $\P_h(\cdot|x, a)$ appears to be linear when we apply it to a value function $V$.
When there is additional problem structure in the feature map $\bphi$ so that $\mathcal{V}$ is relatively small and structured, Eq.~\eqref{eq:weaker_assumption} can potentially provide a usefully weaker condition compared to Assumption \ref{assumption:linear}.

When both the feature map $\bphi$ and the policy $\pi$ are fully generic, we comment that under mild conditions, the assumption of linear transition is then in fact necessary for the Bellman error to be zero for all policies $\pi$.  Indeed, defining the Bellman operator $\mathbb{T}_h^{\pi}$ associated with $\pi$ as 
\begin{equation} \label{eq:stochastic_bellman}
(\mathbb{T}_h^{\pi} Q ) (x,a) = r_h(x,a) +\E_{x' \sim \P_h(\cdot \given x, a)}   \bigl\{  Q(x',  \pi(x') )  \bigl \} , \qquad \forall (x,a) \in \cS \times \cA,
\end{equation}
for any $Q \colon \cS \times \cA \rightarrow \R$, we have the following proposition.
\begin{restatable}{proposition}{PROPbellmanerror} Let $\mathcal{Q} = \{ Q | Q (\cdot, \cdot) = \bphi (\cdot, \cdot)^\top \w, \w\in \R^d \}$ be the family of linear action-value functions. Suppose that $\cS$ is a finite set, and for any $x \in \cS$, there exist two actions  $a , \bar a \in \cA$ such that $\bphi(x,a) \neq \bphi(x, \bar a)$.   
Then, $\mathbb{T}_h ^{\pi} \mathcal{Q} \subset \mathcal{Q}$ for all $\pi$ only if the Markov transition measures $\P_h$ are linear in $\bphi$. 
\end{restatable}

Finally, it remains an interesting future question whether an RL algorithm can be proved to be efficient without assuming a linear structure in the transition dynamics.






\section*{Acknowledgements}
We thank Alekh Agarwal, Zeyuan Allen-Zhu, Sebastian Bubeck, Nan Jiang and Akshay Krishnamurthy for valuable discussions.  This work was supported in part by the DARPA program on Lifelong Learning Machines.  

\bibliographystyle{abbrvnat}
\bibliography{rl_ref}

\newpage

\appendix


\section{Properties of Linear MDP} \label{sec:properties}

In this section, we present some of the basic properties of linear MDPs.

We start with the most important property of a linear MDP: the action-value function is always linear in the feature map $\bphi$ for any policy.

\PROPrealizable*

\begin{proof}
The linearity of the action-value functions directly follows from the Bellman equation in Eq.~\eqref{eq:bellman}:   
\$
Q_h ^{\pi} (x,a ) = r(x, a) + (\P_h V^{\pi} _{h+1}  )(x , a) = \la \bphi(x, a), \btheta_h \ra + \int_{\cS} V_{h+1}^{\pi} (x') \cdot \la \bphi(x, a) , \ud \bmu_h(x') \ra  .  
  \$
Therefore, we have $ Q_h^\pi(x, a) = \la\bphi (x, a),  \w^{\pi}_h\ra$ where $\w_h^{\pi}$ is given by $\w_h^{\pi} = \btheta_h + \int_{\cS} V_{h+1}^{\pi} (x' ) ~\ud \bmu_h(x')$.
\end{proof}




Second, we show that, under mild conditions, the assumption of a linear transition is necessary for the Bellman error to be zero for all policies $\pi$.

\PROPbellmanerror*

  \begin{proof} 
  
  For any fixed state $x_0 \in \cS$, by assumption, there exist two actions $a_0$ and $\bar a_0$ such that 
  $\bphi(x_0,a_0) \neq \bphi(x_0, \bar a_0)$. Then there exists $\w_0 \in \R^d$ such that 
  \#\label{eq:phi_diffw}
  \w _0^\top [ \bphi(x_0,a_0)  -  \bphi(x_0, \bar a_0) ] = 1.
  \#
  We define the function $Q_0 (\cdot , \cdot) = \bphi(\cdot, \cdot) ^\top \w_0$. Additionally, let two policies $\pi_1$ and $\pi_2$ satisfy 
  \#\label{eq:two_policy_assumption}
  \pi_1(  x) = \pi_2(  x), ~~\forall x \in \cS \backslash \{ x_0\}, ~~\text{and}~~ \pi_1(x_0) = a_0, ~~\pi_2(x_0) = \bar a_0.
  \#
  Now consider $\mathbb{T}_h ^{\pi_1} Q_0 - \mathbb{T}_h ^{\pi_2} Q_0$ for any $h$. By the definition of Bellman operator in Eq.~\eqref{eq:stochastic_bellman},  for any $(x, a) \in \cS \times \cA$, 
  we have 
  \#\label{eq:some_bellman_thing}
 &  \mathbb{T}_h ^{\pi_1} Q_0 (x,a)- \mathbb{T}_h ^{\pi_2} Q_0 (x,a) = \sum_{x' \in \cS } \P_h(x' \given x, a)    \bigl\{  Q_0\bigl[ x' , \pi_1(x') \bigr ]  - Q_0\bigl [ x' , \pi_2( x')  \bigr ]  \bigr \} \notag \\
  & \qquad =       \P_h(x_0\given x, a)  \cdot   \bigl [ Q_0( x_0 , a_0 ) - Q_0 ( x_0,  \bar x_0) \bigr ]    =  \P_h(x_0\given x, a)  \cdot \bigl[ \bphi( x_0 , a_0 ) - \bphi ( x_0,  \bar x_0) \bigr ] ^ \top \w_0     ,
     \#
     where the second equality holds due to  Eq.~\eqref{eq:two_policy_assumption}.    Thus, by combining 
    Eq.~\eqref{eq:phi_diffw} and Eq.~\eqref{eq:some_bellman_thing}, we have 
     \$
      \mathbb{T}_h ^{\pi_1} Q_0 (x,a)- \mathbb{T}_h ^{\pi_2} Q_0 (x,a) = \P_h(x_0\given x, a) , \qquad \forall (x,a) \in \cS \times \cA.
\$
Since $\mathbb{T}_h ^{\pi} \mathcal{Q} \subset \mathcal{Q}$ for all $\pi$, we know both  $\mathbb{T}_h ^{\pi_1} Q_0 $ and $\mathbb{T}_h ^{\pi_2} Q_0 $ are elements of $\mathcal{Q}$, so is
   $\P_h(x_0\given  \cdot ,\cdot )$, which implies that  $\P_h(x_0\given  \cdot ,\cdot )$ is a linear function of $\bphi (\cdot, \cdot)$. That is, there exists a vector $\bmu(x_0)$ independent of $(x, a)$ so that $\P_h(x_0\given  x ,a ) = \la \bphi(x, a), \bmu(x_0) \ra$ for all $(x, a)$.  Because this holds for all $x_0 \in \cS$, we have $\P_h(\cdot\given  x ,a ) = \la \bphi(x, a), \bmu(\cdot) \ra$. This concludes the proof.
  \end{proof}

Finally, we note Assumption \ref{assumption:linear} also implicitly enforces the following structure on the feature space since $\P_h(\cdot|x, a)$ must be a probability measure over $\cS$ for any $(x, a) \in \cS \times \cA$.

\begin{proposition}\label{prop:linear_structure}
For a linear MDP,  for any $(x, a, h) \in \cS\times \cA \times [H] $, we have 
\begin{align} \label{eq:transition_condition}
\bphi(x,a) \trans \bmu_h(\cS) =  1, \qquad   
\bphi(x,a) \trans \bmu_h(\mathcal{B})\ge  0, \quad \forall \text{~measurable~} \mathcal{B} \subseteq \cS.
\end{align}
\end{proposition}

\begin{proof}
This proposition immediately follows from the fact that $\P_h (\cdot \given  x, a)$ is a    probability measure over $\cS$ for any $(x, a, h) \in \cS \times \cA \times [ H ] $.
\end{proof}

In particular, the first condition in Eq.~\eqref{eq:transition_condition} requires the image of $\bphi$, $\{ \bphi(x, a) | (x, a) \in \cS \times \cA \}$, to be contained in a $(d-1)$-dimensional hyperphane.





\section{Proof of Theorem \ref{thm:main}} \label{app:proof_main}

In this section, we prove Theorem \ref{thm:main}. We first introduce the notation that is used throughout this section. Then, we present lemmas and their proofs. Finally, we combine the lemmas to prove Theorem \ref{thm:main}.

\paragraph{Notation:} Throughout this section, we denote $\Lambda_h^k$, $\w^k_h$, and $Q_h^k$ as the parameters and the Q-value function estimate in episode $k$. 
Denote value function $V_h^k$ as $V_h^k(x) = \max_a Q_h^k(x, a)$. We also denote $\pi_k$ as the greedy policy induced by $\{Q_h^k\}_{h=1}^H$.
To simplify our presentation, we always denote $\bphi^k_h \defeq \bphi(x^{k}_h, a^{k}_h)$.

First, we prove two lemmas which state that the linear weights $\w_h$ in both the action-value functions and Algorithm \ref{algo:LSVI-UCB} are bounded. 

\begin{lemma}[Bound on Weights of Value Functions] \label{lem:wn_policy}
Under Assumption \ref{assumption:linear}, for any fixed policy $\pi$, let $\{\w^{\pi}_h\}_{h\in[H]}$ be the corresponding weights such that $Q_h^\pi(x, a) = \la\bphi (x, a),  \w^{\pi}_h\ra$ for all $(x, a, h) \in \cS\times\cA\times[H]$. Then, we have
\begin{equation*}
\forall h \in [H], \quad \norm{\w^\pi_h} \le 2H\sqrt{d}.
\end{equation*}
\end{lemma}

\begin{proof}
By the  Bellman equation  in Eq.~\eqref{eq:bellman}, we know, for any $h\in[H]$:
\begin{equation*}
\sind{Q}{\pi}{h}(x, a) = (r_h + \P_h \sind{V}{\pi}{h+1})(x, a). 
\end{equation*}
Since MDP is linear, by definition, this gives:
\begin{equation*}
\w^\pi_h = \btheta_h  + \int V^\pi_{h+1}(x') \dd \bmu_h (x').
\end{equation*}
Under the normalization conditions of Assumption \ref{assumption:linear}, the reward at each step is in $[0, 1]$, thus $V^\pi_{h+1}(x') \le H$ for any state $x'$. Therefore, $\norm{\btheta_h} \le \sqrt{d}$, and $\norm{\int V^\pi_{h+1}(x') \dd \bmu_h (x')} \le H\sqrt{d}$, which finishes the proof.
\end{proof}

\begin{lemma}[Bound on Weights in Algorithm]\label{lem:wn_estimate}
For any $(k, h) \in [K] \times [H]$, the weight $\w^k_h$ in Algorithm \ref{algo:LSVI-UCB} satisfies:
\begin{equation*}
\norm{\w^k_h} \le 2H \sqrt{dk/\lambda}.
\end{equation*}
\end{lemma}

\begin{proof}
For any vector $\v \in \R^d$, we have
\begin{align*}
|\v\trans \w^k_h| & = |\v\trans (\Lambda^k_h)^{-1} \sum_{\tau=1}^{k-1} \bphi^\tau_h [r(x^{\tau}_h, a^{\tau}_h) + \max_a Q_{h+1}(x^{\tau}_{h+1}, a)]|\\
& \le \sum_{\tau = 1}^{k-1}  |\v\trans (\Lambda^k_h)^{-1} \bphi^\tau_h| \cdot 2H 
\le \sqrt{ \bigg[ \sum_{\tau = 1}^{k-1}  \v\trans (\Lambda^k_h)^{-1}\v\bigg]  \cdot \biggl [ \sum_{\tau = 1}^{k-1}  (\bphi^\tau_h)\trans (\Lambda^k_h)^{-1}\bphi^\tau_h\bigg] } \cdot 2H\\
& \le 2H \norm{\v}\sqrt{dk/\lambda}, 
\end{align*}
where the last step is due to Lemma \ref{lem:basic_ineq}. The remainder of the proof follows from the fact that 
$\norm{\w^k_h} = \max_{\v:\norm{\v} = 1} |\v\trans \w^k_h|$.
\end{proof}

Second, we present our main concentration lemma, which is crucial in controlling the fluctuations in least-squares value iteration.
\begin{lemma} \label{lem:stochastic_term}
Under the setting of Theorem \ref{thm:main}, let $c_{\beta}$ be the constant in our definition of $\beta$ (i.e., $\beta = c_{\beta} \cdot dH \sqrt{\logt}$).
There exists an absolute constant $C$ that is independent of $c_{\beta}$ such that for any fixed $p\in[0, 1]$, if we let $\fE$ be the event that:
\begin{equation*}
\forall (k, h)\in [K]\times [H]: \quad \norm{\sum_{\tau = 1}^{k-1} \bphi^\tau_h [V^{k}_{h+1}(x^\tau_{h+1}) - \P_h V^{k}_{h+1}(x_h^\tau, a_h^\tau)]}_{(\Lambda^k_h)^{-1}}
\le C\cdot dH\sqrt{\chi}, 
\end{equation*}
where $\chi = \log [2(c_\beta+1)dT/p]$, then $\P(\fE) \ge 1-p/2$.
\end{lemma}
\begin{proof}
For all $(k, h) \in [K] \times [H]$, by Lemma \ref{lem:wn_estimate} we have $\| \w_h^k \| \leq 2 H \sqrt{dk / \lambda } $. 
In addition, by the construction of $\Lambda_{h}^k$, the minimum   eigenvalue of $\Lambda_h^k$ is lower  bounded by $  \lambda$. Thus, by combining Lemmas \ref{lem:self_norm_covering} and \ref{lem:covering_number}, we have for any fixed $\epsilon>0$ that:
\#\label{eq:combine1}
& \norm{\sum_{\tau = 1}^{k-1} \bphi^\tau_h [V^{k}_{h+1}(x^\tau_{h+1}) - \P_h V^{k}_{h+1}(x_h^\tau, a_h^\tau)]}_{(\Lambda^k_h)^{-1}} ^2 \\
& \qquad \leq 4H^2 \left[  \frac{d}{2}\log\bigg(\frac{k+\lambda}{\lambda} \bigg) + d  \log \bigg(1+ \frac{8H \sqrt{dk }} {\epsilon\sqrt{\lambda} } \bigg) + d^2 \log \bigg( 1 +  \frac{8 d^{1/2}\beta^2 } {\epsilon^2\lambda}  \bigg)  + \log\bigg(\frac{2}{p}\bigg) \right] + \frac{8k^2\epsilon^2}{\lambda} . \notag 
\#
Notice that we choose the hyperparameters $\lambda = 1$ and $\beta = C \cdot d H \logt $ where $C$ is an absolute constant. Finally, picking $\epsilon = dH/k$, by Eq.~\eqref{eq:combine1}, there exists a absolute constant $C > 0$ that is independent of $c_\beta$ such that 
\$
\norm{\sum_{\tau = 1}^{k-1} \bphi^\tau_h [V^{k}_{h+1}(x^\tau_{h+1}) - \P_h V^{k}_{h+1}(x_h^\tau, a_h^\tau)]}_{(\Lambda^k_h)^{-1}} ^2  \leq C \cdot d^2  H^2 \log [2(c_\beta+1)dT/p],
\$ 
which concludes the proof. 
\end{proof}

Next, we recursively bound the difference between the value function maintained in Algorithm \ref{algo:LSVI-UCB} (without bonus) and the true value function of any policy $\pi$. We bound this difference using their expected difference at next step, plus a error term. This error term can be upper bounded by our bonus with high probability. This is the key technical lemma in this section.

\begin{lemma}\label{lem:basic_relation}
There exists an absolute constant $c_\beta$ such that for $\beta = c_\beta \cdot dH\sqrt{\logt}$ where $\logt = \log (2dT/p)$, and
for any fixed policy $\pi$, on the event $\fE$ defined in Lemma \ref{lem:stochastic_term}, we have for all $(x, a, h, k) \in \cS\times\cA\times[H]\times[K]$ that:
\begin{equation*}
\la\bphi(x, a), \w^k_h\ra - Q_h^\pi(x, a)  =  \P_h (V^{k}_{h+1} - V^{\pi}_{h+1})(x, a) + \Delta^k_h(x, a),
\end{equation*}
for some $\Delta^k_h(x, a)$ that satisfies $|\Delta^k_h(x, a)| \le \beta \sqrt{\bphi(x, a)\trans (\Lambda^k_h)^{-1}  \bphi(x, a)}$.
\end{lemma}

\begin{proof}
By Proposition \ref{prop:realizable} and the Bellman equation Eq.~\eqref{eq:bellman}, we know for any $(x, a, h) \in \cS \times \cA \times [H]$:

\begin{equation*}
Q^\pi_h(x, a) \defeq  \la \bphi(x, a), \w^\pi_h\ra = (r_h + \P_h \sind{V}{\pi}{h+1})(x, a). 
\end{equation*}
This gives:
\begin{align*}
\w^k_h -  \w^\pi_h
&= (\Lambda^k_h)^{-1} \sum_{\tau = 1}^{k-1} \bphi^\tau_h [r^\tau_h + V^{k}_{h+1}(x^\tau_{h+1})]- \w^\pi_h  \\
&= (\Lambda^k_h)^{-1} \bigg \{-\lambda \w^\pi_h + \sum_{\tau = 1}^{k-1} \bphi^\tau_h \bigl [V^{k}_{h+1}(x^\tau_{h+1}) - \P_h \sind{V}{\pi}{h+1}(x_h^\tau, a_h^\tau) \bigr ]\bigg\} \\
&= \underbrace{-\lambda (\Lambda^k_h)^{-1} \w^\pi_h}_{\q_1} + 
\underbrace{(\Lambda^k_h)^{-1} \sum_{\tau = 1}^{k-1} \bphi^\tau_h \big [V^{k}_{h+1}(x^\tau_{h+1}) - \P_h V^{k}_{h+1}(x_h^\tau, a_h^\tau)\big ]}_{\q_2} \\
&\quad + \underbrace{(\Lambda^k_h)^{-1}\sum_{\tau = 1}^{k-1} \bphi^\tau_h \P_h (V^{k}_{h+1} - V^\pi_{h+1})(x_h^\tau, a_h^\tau)}_{\q_3}. 
\end{align*}
Now, we bound the terms on the right-hand side individually.
For the first term,
\begin{equation*}
|\la \bphi(x, a), \q_1\ra| = |\lambda \la \bphi(x, a), (\Lambda^k_h)^{-1} \w^\pi_h\ra|
\le \sqrt{\lambda} \norm{\w_h^\pi} \sqrt{\bphi(x, a)\trans (\Lambda^k_h)^{-1}  \bphi(x, a)}. 
\end{equation*}
For the second term, given the event $\fE$ defined in Lemma \ref{lem:stochastic_term}, we have:
\begin{equation*}
|\la \bphi(x, a), \q_2\ra| \le c_0\cdot dH\sqrt{\chi} \sqrt{\bphi(x, a)\trans (\Lambda^k_h)^{-1}  \bphi(x, a)} 
\end{equation*}
for an absolute constant $c_0$ independent of $c_\beta$, and $\chi = \log [2(c_\beta+1)dT/p]$. For the third term,
\begin{align*}
\la \bphi(x, a), \q_3\ra &= \bigg \la \bphi(x, a), (\Lambda^k_h)^{-1}\sum_{\tau = 1}^{k-1} \bphi^\tau_h \P_h (V^{k}_{h+1} - V^\pi_{h+1})(x_h^\tau, a_h^\tau) \bigg \ra\\
&=\bigg \la \bphi(x, a), (\Lambda^k_h)^{-1}\sum_{\tau = 1}^{k-1} \bphi^\tau_h (\bphi^\tau_h)\trans \int (V^{k}_{h+1} - V^\pi_{h+1})(x') \dd \bmu_h(x')\bigg\ra\\
&= \underbrace{\bigg \la \bphi(x, a), \int (V^{k}_{h+1} - V^\pi_{h+1})(x') \dd \bmu_h(x')\bigg \ra}_{p_1}
\underbrace{-\lambda \bigg\la \bphi(x, a), (\Lambda^k_h)^{-1}\int (V^{k}_{h+1} - V^\pi_{h+1})(x') \dd \bmu_h(x') \bigg\ra}_{p_2}, 
\end{align*}
where, by Eq.~\eqref{eq:linear_transition}, we have
\begin{align*}
p_1 = \P_h (V^{k}_{h+1} - V^\pi_{h+1})(x, a),  \qquad 
|p_2|  
\le 2 H \sqrt{d\lambda} \sqrt{\bphi(x, a)\trans (\Lambda^k_h)^{-1}  \bphi(x, a)}.
\end{align*}
Finally, since $\la\bphi(x, a), \w^k_h\ra - Q_h^\pi(x, a) = \la \bphi(x, a), \w^k_h - \w^\pi_h\ra =\la \bphi(x, a), \q_1 + \q_2 + \q_3\ra$, by Lemma \ref{lem:wn_policy} and our choice of parameter $\lambda$, we have
\begin{equation*}
|\la\bphi(x, a), \w^k_h\ra - Q_h^\pi(x, a)  -  \P_h (V^{k}_{h+1} - V^{\pi}_{h+1})(x, a)| \le c' \cdot dH\sqrt{\chi} \sqrt{\bphi(x, a)\trans (\Lambda^k_h)^{-1}  \bphi(x, a)},
\end{equation*}
for an absolute constant $c'$ independent of $c_{\beta}$. Finally, to prove this lemma, we only need to show that there exists a choice of absolute constant $c_\beta$ so that 
\begin{equation} \label{eq:choice_beta_constant}
c' \sqrt{\iota + \log(c_\beta + 1)} \le c_\beta \sqrt{\iota},
\end{equation}
where $\iota = \log (2dT/p)$. We know $\iota \in [\log 2, \infty)$ by its definition, and $c'$ is an absolute constant independent of $c_\beta$. Therefore, we can pick an absolute constant $c_\beta$ which satisfies 
$c'\sqrt{\log 2 + \log(c_\beta + 1)} \le c_\beta \sqrt{\log 2}$. This choice of $c_\beta$ will make Eq.~\eqref{eq:choice_beta_constant} hold for all $\iota \in [\log 2, \infty)$, which finishes the proof.
\end{proof}

Lemma \ref{lem:basic_relation} implies that by adding appropriate bonuses, $Q^k_h$ in Algorithm \ref{algo:LSVI-UCB} can be always an upper bound of $Q^\star_h$ with high confidence. 

\begin{lemma}[UCB] \label{lem:UCB}
Under the setting of Theorem \ref{thm:main}, on the event $\fE$ defined in Lemma \ref{lem:stochastic_term}, we have $Q^k_h(x, a) \ge Q^\star_h(x, a)$ for all $(x, a, h, k) \in \cS\times\cA\times[H]\times[K]$.
\end{lemma}

\begin{proof}
We prove this lemma by induction. 

First, we prove the base case, at the last step $H$.  The statement holds because $Q^k_H(x, a) \ge Q^\star_H(x, a)$. Since the value function at $H+1$ step is zero, by Lemma \ref{lem:basic_relation}, we have:
\begin{equation*}
|\la\bphi(x, a), \w^k_H\ra - Q_H^\star(x, a)| \le \beta \sqrt{\bphi(x, a)\trans (\Lambda^k_H)^{-1}  \bphi(x, a)}.
\end{equation*}
Therefore, we know:
\begin{equation*}
Q_H^\star(x, a) \le \min\{\la\bphi(x, a), \w^k_H\ra + \beta\sqrt{\bphi(x, a)\trans (\Lambda^k_H)^{-1}  \bphi(x, a)}, H\} = Q^k_H(x, a).
\end{equation*}
Now, suppose the statement holds true at step $h+1$ and consider step $h$. Again, by Lemma \ref{lem:basic_relation}, we have:
\begin{equation*}
|\la\bphi(x, a), \w^k_h\ra - Q_h^\star(x, a) -  \P_h (V^{k}_{h+1} - V^{\star}_{h+1})(x, a)| \le \beta \sqrt{\bphi(x, a)\trans (\Lambda^k_h)^{-1}  \bphi(x, a)}.
\end{equation*}
By the induction assumption that $\P_h (V^{k}_{h+1} - V^{\star}_{h+1})(x, a) \ge 0$, we have:
\begin{equation*}
Q_h^\star(x, a) \le \min\{\la\bphi(x, a), \w^k_h\ra + \beta\sqrt{\bphi(x, a)\trans (\Lambda^k_h)^{-1}  \bphi(x, a)}, H\} = Q^k_h(x, a),
\end{equation*}
which concludes the proof. 
\end{proof}

Lemma \ref{lem:basic_relation} also easily transforms to a recursive formula for $\delta^k_h = V^k_h(x^{k}_{h}) - V^{\pi_k}_h(x^{k}_{h})$. This formula will be very useful in proving the main theorem.

\begin{lemma}[Recursive formula]\label{lem:recursive}
Let $\delta^k_h = V^k_h(x^{k}_{h}) - V^{\pi_k}_h(x^{k}_{h})$, and $\zeta^k_{h+1} = 
\E[\delta^k_{h+1}|x^{k}_h, a^{k}_h] - \delta_{h+1}^k$. Then, on the event $\fE$ defined in Lemma \ref{lem:stochastic_term}, we have the following for any $(k, h)\in [K]\times [H]$:
\begin{equation*}
\delta^k_h \le \delta^k_{h+1} + \zeta^k_{h+1} + 2\beta\sqrt{(\bphi^{k}_h)\trans (\Lambda^k_h)^{-1}  \bphi^{k}_h}. 
\end{equation*}
\end{lemma}

\begin{proof} By Lemma \ref{lem:basic_relation} we have that for any $(x, a, h, k) \in \cS \times \cA \times [H] \times [K]$:
\begin{equation*}
Q^k_h(x, a) - Q^{\pi_k}_h(x, a)
\le \P_h (V^{k}_{h+1} - V^{\pi_k}_{h+1})(x, a)
+ 2\beta \sqrt{\bphi(x, a)\trans (\Lambda^k_h)^{-1}  \bphi(x, a)}
\end{equation*}
and finally, by Algorithm \ref{algo:LSVI-UCB} and the definition of $V^{\pi_k}$ we have
$$\delta^k_h = Q^k_h(x^{k}_h, a^{k}_h) - Q^{\pi_k}_h(x^{k}_h, a^{k}_h), $$
which finishes the proof.
\end{proof}

Finally, we are ready to prove the main theorem. We restate our main theorem as follows.

\THMmain*

\begin{proof}

We use the notion of $\delta^k_h$ and $\zeta_{h}^k$ as in Lemma \ref{lem:recursive}.
	We condition on the event  $\fE$ defined in Lemma \ref{lem:stochastic_term} with $\delta = p/2$.  By   Lemmas \ref{lem:UCB} and   \ref{lem:recursive}, we have
\begin{align}\label{eq:final1}
\text{Regret}(K) & =  \sum_{k=1}^{K} \left[\sind{V}{\star}{1} (\ind{x}{k}{1}) - \sind{V}{\pi_k}{1} (\ind{x}{k}{1})\right]
\le \sum_{k=1}^{K} \delta^k_1 \le \sum_{k=1}^{K}\sum_{h=1}^H  \zeta^k_{h} + 2\beta \sum_{k=1}^{K}\sum_{h=1}^H \sqrt{(\bphi^{k}_h)\trans (\Lambda^k_h)^{-1}  \bphi^{k}_h}.
\end{align}
We now bound the two terms on the right-hand side of Eq.~\eqref{eq:final1} separately. 
For the first term, since the computation of $V_h^k$ is independent of the new observation $x^k_h$ at episode $k$, we obtain that $\{ \zeta^k_{h}  \}$ is a martingale difference sequence satisfying $| \zeta^k_{h} | \le 2H$ for all $(k,h)$. 
Therefore, by the Azuma-Hoeffding inequality, for any $t > 0$, we have 
\$
\P \biggl (   \sum_{k=1}^{K}\sum_{h=1}^H  \zeta^k_{h}   > t \biggr ) \geq \exp \bigg (   \frac{- t^2 } { 2 T \cdot H^2 } \bigg ).
\$
Hence, with probability at least $1 -p / 2$,  we have 
\#\label{eq:final2}
  \sum_{k=1}^{K}\sum_{h=1}^H  \zeta^k_{h}  \leq  \sqrt{2 T H^2 \cdot \log ( 2 / p)} \leq 2 H  \sqrt{ T\logt}, 
  \#
  where $\logt = \log (2dT/p)$. 
  Furthermore, for the second term, note that the minimum eigenvalue of $\Lambda_{h}^k$ is at least $\lambda$ (which equals to 1) for all $(k, h) \in [K] \times [H]$. Also notice that $\| \bphi_h^k \| \leq 1$. 
  By Lemma \ref{lemma:telescope}, for any $h \in [H]$, we have
 \begin{equation*}
 \sum_{k =1}^K   (\bphi^{k}_h)\trans (\Lambda^k_h)^{-1}  \bphi^{k}_h \leq  
  2\log \left[\frac{\det (\Lambda_h^{k+1} )}{\det (\Lambda_h^1)}\right]. 
 \end{equation*}
  Moreover, note that $\|\Lambda^{k+1}_h\| = \|\sum_{\tau=1}^k \bphi^k_h (\bphi^k_h) \trans + \lambda \I\| \le \lambda + k$; this gives
  \#\label{eq:final08}
  \sum_{k =1}^K   (\bphi^{k}_h)\trans (\Lambda^k_h)^{-1}  \bphi^{k}_h  \le 2d \log \left[\frac{\lambda + k}{\lambda}\right]
  \le 2d \logt. 
  \#
  Now, by the Cauchy-Schwartz inequality, we have
  \#\label{eq:final8}
  \sum_{k=1}^{K} \sum_{h=1}^H  \sqrt{(\bphi^{k}_h)\trans (\Lambda^k_h)^{-1}  \bphi^{k}_h}   \leq \sum_{h=1}^H  \sqrt{K} \cdot \biggl [ \sum_{k=1}^K (\bphi^{k}_h)\trans (\Lambda^k_h)^{-1}  \bphi^{k}_h \biggl] ^{1/2 } \leq  H \cdot \sqrt{ 2 dK \logt} , 
  \#
  which yields an upper bound on the second term in Eq.~\eqref{eq:final1}. Finally, combining Eq.~\eqref{eq:final1}, Eq.~\eqref{eq:final2}, Eq.~\eqref{eq:final8}, and with our choice of $\beta = c \cdot dH\sqrt{\logt}$ for some absolute constant c, we conclude that with probability $1-p$:
  \$
  \text{Regret}(K)  \leq 2 H \sqrt{T \logt} +  \beta H\sqrt{ 2 dK \logt}  \le c' \cdot \sqrt{d^3 H^3 T \logt^2},
  \$
  for some absolute constant $c'$. This concludes the proof.
  \end{proof}

\section{Proof of Theorem \ref{thm:main_mis}} \label{app:proof_main_mis}
In this section, we prove Theorem \ref{thm:main_mis}. At a high level, the proof structure is similar to the structure in Appendix \ref{app:proof_main}. We will particularly focus on the parts that require different treatments in the misspecified setting.

\paragraph{Notation:} Throughout this section, we denote $\Lambda_h^k$, $\w^k_h$, and $Q_h^k$ as the parameters and the Q-value functions estimated in episode $k$. 
Denote the value function $V_h^k$ as $V_h^k(x) = \max_a Q_h^k(x, a)$. We denote $\pi_k$ as the greedy policy induced by $\{Q_h^k\}_{h=1}^H$.
To simplify the presentation, we denote $\bphi^k_h \defeq \bphi(x^{k}_h, a^{k}_h)$.

First, we establish a lemma that is the counterpart of Lemma \ref{prop:realizable} in the misspecified setting: for any policy $\pi$, its action-value function is always close to a linear function.

\begin{lemma}\label{lem:misspecified_error}
For a $\zeta$-nearly linear MDP, for any policy $\pi$, there exist corresponding weights $\{\w^{\pi}_h\}_{h\in[H]}$ where $\w^{\pi}_h = \btheta_h  + \int V^\pi_{h+1}(x') \dd \bmu_h (x')$ such that for any $(x, a, h) \in \cS\times\cA\times[H]$:
\begin{equation*}
|Q_h^\pi(x, a) - \la\bphi (x, a),  \w^{\pi}_h\ra| \le 2H\zeta.
\end{equation*}
\end{lemma}
\begin{proof} Since $\bmu_h$ and $\btheta_h$ satisfy Eq.\eqref{eq:nearly_linear_transition}, we have:
\begin{align*}
&|Q^\pi_h(x, a) -  \la \bphi(x, a), \w^\pi_h\ra| \\
& \qquad \le  |r_h(x, a) - \la \bphi(x, a), \btheta_h\ra| + |\P_h \sind{V}{\pi}{h+1}(x, a) - \la \bphi(x, a),\int V^\pi_{h+1}(x') \dd \bmu_h (x') \ra| \\
& \qquad \le  \zeta + H\zeta \le 2H\zeta , 
\end{align*}
which finishes the proof.
\end{proof}

We can again show that the linear weights defined in Lemma \ref{lem:misspecified_error} are bounded.
\begin{lemma}[Bound on Weights of Value Functions] \label{lem:wn_policy_mis}
Under Assumption \ref{assumption:nearly_linear}, for any policy $\pi$, let $\{\w^{\pi}_h\}_{h\in[H]}$ be the corresponding weights as defined in Lemma \ref{lem:misspecified_error}.
Then, we have
\begin{equation*}
\forall h \in [H], \quad \norm{\w^\pi_h} \le 2H\sqrt{d}.
\end{equation*}
\end{lemma}

\begin{proof}
Under the normalization conditions of Assumption \ref{assumption:nearly_linear}, the reward at each step is in $[0, 1]$, thus $V^\pi_{h+1}(x') \le H$ for any state $x'$. Therefore, $\norm{\btheta_h} \le \sqrt{d}$, and $\norm{\int V^\pi_{h+1}(x') \dd \bmu_h (x')} \le H\sqrt{d}$, which finishes the proof.
\end{proof}

Similar to Lemma \ref{lem:stochastic_term}, we also bound the stochastic noise in concentration.
\begin{lemma} \label{lem:stochastic_term_mis}
Under the setting of Theorem \ref{thm:main_mis}, let $c_{\beta}$ be the constant in our choice of $\beta_k$ (i.e. $\beta_k = c_\beta \cdot (d\sqrt{\logt}+\zeta \sqrt{kd})H$),
There exists an absolute constant $C$ that is independent of $c_{\beta}$ such that for any fixed $p\in[0, 1]$, if we let $\fE$ be the event that:
\begin{equation*}
\forall (k, h)\in [K]\times [H]: \quad \norm{\sum_{\tau = 1}^{k-1} \bphi^\tau_h [V^{k}_{h+1}(x^\tau_{h+1}) - \P_h V^{k}_{h+1}(x_h^\tau, a_h^\tau)]}_{(\Lambda^k_h)^{-1}}
\le C\cdot dH\sqrt{\chi}, 
\end{equation*}
where $\chi = \log [2(c_\beta+1)dT/p]$, then $\P(\fE) \ge 1-p/2$.
\end{lemma}
\begin{proof}
The proof is essentially the same as the proof for Lemma \ref{lem:stochastic_term}, with the only difference that $\beta_k$ is now bounded by $c_\beta (d\sqrt{\logt} +\zeta\sqrt{Kd})H$ instead of $c_\beta dH\sqrt{\logt}$ as in Lemma \ref{lem:stochastic_term}. Because $\zeta \le 1$ as in Assumption \ref{assumption:nearly_linear}, the new bound of $\beta_k$ only affects the choice of absolute $C$ in Lemma \ref{lem:stochastic_term_mis}. 
\end{proof}




In the misspecified case, we also need to bound an error term where the noise can be potentially adversarial instead of stochastic. The adversarial noise is precisely due to model misspecification.
\begin{lemma}\label{lem:infinite_error}
Let $\{\epsilon_\tau\}$ be any sequence so that $|\epsilon_\tau| \le B$ for any $\tau$. Then, we have for any $(h, k) \in [H]\times[K]$ and any $\bphi \in \R^d$ that:
\begin{equation*}
\bigg |\bphi\trans (\Lambda^k_h)^{-1} \sum_{\tau = 1}^{k-1} \bphi^\tau_h \epsilon_\tau \bigg|
\le B\sqrt{dk\bphi\trans (\Lambda^k_h)^{-1}  \bphi}. 
\end{equation*}
\end{lemma}

\begin{proof}
By the Cauchy-Schwarz inequality,
\begin{align*}
\bigg |\bphi\trans (\Lambda^k_h)^{-1} \sum_{\tau = 1}^{k-1} \bphi^\tau_h \epsilon_\tau \bigg| \le& \sum_{\tau = 1}^{k-1}  |\bphi\trans (\Lambda^k_h)^{-1} \bphi^\tau_h| \cdot B
\le \sqrt{\biggl [ \sum_{\tau = 1}^{k-1}  \bphi\trans (\Lambda^k_h)^{-1}\bphi\biggr ]  \cdot \biggl [ \sum_{\tau = 1}^{k-1}  (\bphi^\tau_h)\trans (\Lambda^k_h)^{-1}\bphi^\tau_h\biggr ] } \cdot B\\
\le& B \sqrt{dk\bphi\trans (\Lambda^k_h)^{-1}  \bphi},
\end{align*}
where the last inequality is due to Lemma \ref{lem:basic_ineq}.
\end{proof}

Now we are ready to prove the key lemma, which is the counterpart of Lemma \ref{lem:basic_relation}.

\begin{lemma}\label{lem:basic_relation_mis}There exists an absolute constant $c_\beta$ such that for $\beta_k = c_\beta \cdot (d\sqrt{\logt}+\zeta \sqrt{kd})H$ where $\logt = \log (2dT/p)$, and
for any fixed policy $\pi$, on the event $\fE$ defined in Lemma \ref{lem:stochastic_term_mis}, we have for all $(x, a, h, k) \in \cS\times\cA\times[H]\times[K]$ that:
\begin{equation*}
\la\bphi(x, a), \w^k_h\ra - Q_h^\pi(x, a)  =  \P_h (V^{k}_{h+1} - V^{\pi}_{h+1})(x, a) + \Delta^k_h(x, a),
\end{equation*}
for some $\Delta^k_h(x, a)$ that satisfies $|\Delta^k_h(x, a)| \le \beta_k \sqrt{\bphi(x, a)\trans (\Lambda^k_h)^{-1}  \bphi(x, a)} + 4H\zeta$.
\end{lemma}

\begin{proof}
By Lemma \ref{lem:misspecified_error}, there exists $\w^\pi_h = \btheta_h  + \int V^\pi_{h+1}(x') \dd \bmu_h (x')$ so that
for any $(x, a, h) \in \cS\times \cA \times [H]$:
\begin{equation*}
|Q_h^\pi(x, a) - \la\bphi (x, a),  \w^{\pi}_h\ra| \le 2H\zeta.
\end{equation*}
On the other hand, let $\tilde{\P}(\cdot|x, a) = \la\bphi(x, a), \bmu_h(\cdot)\ra$. Then, for any $(x, a) \in \cS \times \cA$, we have
\begin{equation*}
\la \bphi(x, a), \w^\pi_h\ra = \la \bphi(x, a), \btheta_h\ra + \tilde{\P}_h V^\pi_{h+1}(x, a).
\end{equation*}
This further gives
\begin{align*}
\lefteqn{\w^k_h -  \w^\pi_h
= (\Lambda^k_h)^{-1} \sum_{\tau = 1}^{k-1} \bphi^\tau_h [r^\tau_h + V^{k}_{h+1}(x^\tau_{h+1})]- \w^\pi_h}  \\
&= (\Lambda^k_h)^{-1} \bigg \{-\lambda \w^\pi_h + \sum_{\tau = 1}^{k-1} \bphi^\tau_h \bigl [r^{\tau}_h + V^{k}_{h+1}(x^\tau_{h+1}) - (\bphi^\tau_h)\trans\btheta_h - \tilde{\P}_h \sind{V}{\pi}{h+1}(x_h^\tau, a_h^\tau) \bigl ] \bigg\} \\
&= \underbrace{-\lambda (\Lambda^k_h)^{-1} \w^\pi_h}_{\q_1} + 
\underbrace{(\Lambda^k_h)^{-1} \sum_{\tau = 1}^{k-1} \bphi^\tau_h \bigl  [V^{k}_{h+1}(x^\tau_{h+1}) - \P_h V^{k}_{h+1}(x_h^\tau, a_h^\tau) \bigr ]}_{\q_2} 
 + \underbrace{(\Lambda^k_h)^{-1}\sum_{\tau = 1}^{k-1} \bphi^\tau_h \tilde{\P}_h (V^{k}_{h+1} - V^\pi_{h+1})(x_h^\tau, a_h^\tau)}_{\q_3}\\
  & \ \ + \underbrace{(\Lambda^k_h)^{-1}\sum_{\tau = 1}^{k-1} \bphi^\tau_h  \bigl [r^{\tau}_h - (\bphi^\tau_h)\trans\btheta_h + (\P_h - \tilde{\P}_h) V^k_{h+1}(x_h^\tau, a_h^\tau) \bigr ]}_{\q_4}. 
\end{align*}
Now, we bound the terms on the right-hand side individually.
For the first term,
\begin{equation*}
|\la \bphi(x, a), \q_1\ra| = |\lambda \la \bphi(x, a), (\Lambda^k_h)^{-1} \w^\pi_h\ra|
\le \sqrt{\lambda} \norm{\w_h^\pi} \sqrt{\bphi(x, a)\trans (\Lambda^k_h)^{-1}  \bphi(x, a)}. 
\end{equation*}
For the second term, given the event $\fE$ defined in Lemma \ref{lem:stochastic_term_mis}, we have:
\begin{equation*}
|\la \bphi(x, a), \q_2\ra| \le c_0 \cdot dH \sqrt{\chi} \sqrt{\bphi(x, a)\trans (\Lambda^k_h)^{-1}  \bphi(x, a)}, 
\end{equation*}
for an absolute constant $c_0$ independent of $c_\beta$, and $\chi = \log [2(c_\beta+1)dT/p]$. For the third term,
\begin{align*}
\la \bphi(x, a), \q_3\ra &= \la \bphi(x, a), (\Lambda^k_h)^{-1}\sum_{\tau = 1}^{k-1} \bphi^\tau_h \tilde{\P}_h (V^{k}_{h+1} - V^\pi_{h+1})(x_h^\tau, a_h^\tau)\ra\\
&= \la \bphi(x, a), (\Lambda^k_h)^{-1}\sum_{\tau = 1}^{k-1} \bphi^\tau_h (\bphi^\tau_h)\trans \int (V^{k}_{h+1} - V^\pi_{h+1})(x') \dd \bmu_h(x')\ra\\
&= \underbrace{\la \bphi(x, a), \int (V^{k}_{h+1} - V^\pi_{h+1})(x') \dd \bmu_h(x')\ra}_{p_1}
\underbrace{-\lambda \la \bphi(x, a), (\Lambda^k_h)^{-1}\int (V^{k}_{h+1} - V^\pi_{h+1})(x') \dd \bmu_h(x') \ra}_{p_2}, 
\end{align*}
where by definition of $\tilde{\P}_h$, we have
\begin{align*}
p_1   = \tilde{\P}_h (V^{k}_{h+1} - V^\pi_{h+1})(x, a) , \qquad 
|p_2|  
\le 2 H \sqrt{d\lambda} \sqrt{\bphi(x, a)\trans (\Lambda^k_h)^{-1}  \bphi(x, a)}.
\end{align*}
Since $\|\tilde{\P}_h - \P_h\|_{\mathrm{TV}} \le \zeta$, we have 
\begin{equation*}
|p_1 - \P_h (V^{k}_{h+1} - V^\pi_{h+1})(x, a)| = |(\P_h-\tilde{P}_h) (V^{k}_{h+1} - V^\pi_{h+1})(x, a)| \le 2H\zeta.
\end{equation*}
For the fourth term, by Lemma \ref{lem:infinite_error}, we have
\begin{equation*}
|\la \bphi(x, a), \q_4\ra| \le 2H\zeta\sqrt{dk \bphi(x, a)\trans (\Lambda^k_h)^{-1}  \bphi(x, a)}. 
\end{equation*}

Finally, since $\la \bphi(x, a), \w^k_h - \w^\pi_h\ra =\la \bphi(x, a), \q_1 + \q_2 + \q_3 + \q_4\ra$, we have:
\begin{equation*}
|\la\bphi(x, a), \w^k_h\ra - Q_h^\pi(x, a)  -  \P_h (V^{k}_{h+1} - V^{\pi}_{h+1})(x, a)| \le (c' d\sqrt{\chi} + 2\zeta\sqrt{kd})H \sqrt{\bphi(x, a)\trans (\Lambda^k_h)^{-1}  \bphi(x, a)} + 4H\zeta,
\end{equation*}
for an absolute constant $c'$ independent of $c_{\beta}$. As in the proof of Lemma \ref{lem:basic_relation}, to prove this lemma, we only need to show that there exists a choice of absolute constant $c_\beta$ so that $c_\beta \ge 2$, and
\begin{equation*}
c' \sqrt{\iota + \log(c_\beta + 1)} \le c_\beta \sqrt{\iota}  \text{~for any~}\iota \in [\log 2, \infty).
\end{equation*}
This can be done by an picking absolute constant $c_\beta$ that satisfies $c'\sqrt{\log 2 + \log(c_\beta + 1)} \le c_\beta \sqrt{\log 2}$.
\end{proof}

Given Lemma \ref{lem:basic_relation_mis}, we can now easily proceed to prove that $Q^k_h$ is a upper bound of $Q^\star_h$ up to an error that depends linearly on the misspecification $\zeta$.

\begin{lemma}[UCB] \label{lem:UCB_misspecified}
Under the setting of Theorem \ref{thm:main_mis}, on the event $\fE$ defined in Lemma \ref{lem:stochastic_term_mis}, we have $Q^k_h(x, a) \ge Q^\star_h(x, a) - 4H(H+1-h)\zeta$ for all $(x, a, h, k) \in \cS\times\cA\times[H]\times[K]$.
\end{lemma}

\begin{proof}
We prove this lemma by induction. 

First, we consider the base case. The statement holds for the last step $H$, i.e., $Q^k_H(x, a) \ge Q^\star_H(x, a) - 4H\zeta$. Since the value function at $H+1$ step is zero, by Lemma \ref{lem:basic_relation_mis}, we have:
\begin{equation*}
|\la\bphi(x, a), \w^k_H\ra - Q_H^\star(x, a)| \le \beta_k \sqrt{\bphi(x, a)\trans (\Lambda^k_H)^{-1}  \bphi(x, a)} + 4H\zeta. 
\end{equation*}
Therefore, we obtain that 
\begin{equation*}
Q_H^\star(x, a) - 4H\zeta \le \min\{\la\bphi(x, a), \w^k_H\ra + \beta_k\sqrt{\bphi(x, a)\trans (\Lambda^k_H)^{-1}  \bphi(x, a)}, H\} = Q^k_H(x, a). 
\end{equation*}
Now, suppose the statement holds true at step $h+1$, and consider step $h$. Again, by Lemma \ref{lem:basic_relation_mis}, we have:
\begin{equation*}
|\la\bphi(x, a), \w^k_h\ra - Q_h^\star(x, a) -  \P_h (V^{k}_{h+1} - V^{\star}_{h+1})(x, a)| \le \beta_k \sqrt{\bphi(x, a)\trans (\Lambda^k_h)^{-1}  \bphi(x, a)}
+ 4H\zeta. 
\end{equation*}
By the induction assumption that $\P_h (V^{k}_{h+1} - V^{\star}_{h+1})(x, a) \ge -4H(H-h)\zeta$, we have:
\begin{equation*}
Q_h^\star(x, a) - 4H(H+1-h)\zeta \le \min\{\la\bphi(x, a), \w^k_h\ra + \beta_k\sqrt{\bphi(x, a)\trans (\Lambda^k_h)^{-1}  \bphi(x, a)}, H\} = Q^k_h(x, a).
\end{equation*}
Therefore, we conclude the proof of this lemma.
\end{proof}

The gap $\delta^k_h = V^k_h(x^{k}_{h}) - V^{\pi_k}_h(x^{k}_{h})$ also has a recursive formula similar to Lemma \ref{lem:recursive}.

\begin{lemma}[Recursive formula]\label{lem:recursive_mis}
Let $\delta^k_h = V^k_h(x^{k}_{h}) - V^{\pi_k}_h(x^{k}_{h})$, and $\zeta^k_{h+1} = 
\E[\delta^k_{h+1}|x^{k}_h, a^{k}_h] - \delta_{h+1}^k$. Then, on the event $\fE$ defined in Lemma \ref{lem:stochastic_term_mis}, we have the following for any $(k, h)\in [K]\times [H]$:
\begin{equation*}
\delta^k_h \le \delta^k_{h+1} + \zeta^k_{h+1} + 2\beta_k\sqrt{(\bphi^{k}_h)\trans (\Lambda^k_h)^{-1}  \bphi^{k}_h}.
\end{equation*}
\end{lemma}

\begin{proof} This is because by Lemma \ref{lem:basic_relation_mis}, we have for any $(x, a, h, k) \in \cS \times \cA \times [H] \times [K]$:
\begin{equation*}
Q^k_h(x, a) - Q^{\pi_k}_h(x, a)
\le \P_h (V^{k}_{h+1} - V^{\pi_k}_{h+1})(x, a)
+ 2\beta_k \sqrt{\bphi(x, a)\trans (\Lambda^k_h)^{-1}  \bphi(x, a)} + 4H\zeta. 
\end{equation*}
Finally, by Algorithm \ref{algo:LSVI-UCB} and the definition of $V^{\pi_k}$ we have
$$\delta^k_h = Q^k_h(x^{k}_h, a^{k}_h) - Q^{\pi_k}_h(x^{k}_h, a^{k}_h), $$
which finishes the proof.
\end{proof}

Finally, we are ready to combine all previous lemmas to prove the main theorem in the misspecified setting.

\THMmainmis*

\begin{proof}[Proof of Theorem \ref{thm:main_mis}]
	The proof of this theorem is similar to that of  Theorem \ref{thm:main}. 
	We condition on the event  $\fE$ defined in Lemma \ref{lem:stochastic_term_mis}.
	For for any $(k,h) \in [K]\times [H]$, 
	we define $\delta^k_h = V^k_h(x^{k}_{h}) - V^{\pi_k}_h(x^{k}_{h})$.   
	By Lemma \ref{lem:UCB_misspecified}, we have 
	$Q_1^k (x,a) \geq Q_1^*(x,a ) -4H^2  \zeta $ for all $k\in [K]$, which implies that $\sind{V}{\star}{1} (\ind{x}{k}{1}) - \sind{V}{\pi_k}{1} (\ind{x}{k}{1}) \leq \delta_1^k + 4H^2  \zeta$.
	Furthermore, by  Lemma \ref{lem:recursive_mis}, on the event $\fE$ we have:
\begin{align}\label{eq:mis1}
\text{Regret}(K) & =  \sum_{k=1}^{K} \left[\sind{V}{\star}{1} (\ind{x}{k}{1}) - \sind{V}{\pi_k}{1} (\ind{x}{k}{1})\right]
\le \sum_{k=1}^{K} [ \delta^k_1 + 4H^2 \zeta] \nn \\
& \le \sum_{k=1}^{K}\sum_{h=1}^H  \zeta^k_{h} + 2\sum_{k=1}^{K}\beta_k\sum_{h=1}^H \sqrt{(\bphi^{k}_h)\trans (\Lambda^k_h)^{-1}  \bphi^{k}_h}
+ 4HT\zeta,
\end{align}
where we use the fact that $T = HK$.
Since $\{ \zeta^k_{h} \}$  is a   martingale difference sequence with each term bounded by $2H$,
 the Azuma-Hoeffding inequality implies that 
 \begin{equation}\label{eq:mis2}
 \sum_{k=1}^{K}\sum_{h=1}^H  \zeta^k_{h}  \leq  \sqrt{2 T H^2 \cdot \log ( 2 / p)} \leq 2 H  \sqrt{ T\logt}
 \end{equation} 
 holds with  probability at least $1 -p / 2$, where $\logt = \log (2dT/p)$. 
 Moreover, by the Cauchy-Schwarz inequality, we have 
 \begin{equation}
 \label{eq:mis3}
  \sum_{k=1}^{K} \beta_k  \sqrt{(\bphi^{k}_h)\trans (\Lambda^k_h)^{-1}  \bphi^{k}_h}   \leq \biggl [ \sum_{k=1}^K \beta_k^2 \biggr ] ^{1/2} \cdot \biggl [ \sum_{k=1}^K (\bphi^{k}_h)\trans (\Lambda^k_h)^{-1}  \bphi^{k}_h \biggl] ^{1/2 }  . 
 \end{equation}
Similarly to Eq.~\eqref{eq:final08}, we have 
\begin{equation}
\label{eq:mis4}
 \biggl [ \sum_{k=1}^K (\bphi^{k}_h)\trans (\Lambda^k_h)^{-1}  \bphi^{k}_h \biggl] ^{1/2 }  \leq \sqrt{ 2 d \logt}. 
\end{equation}
 Moreover, since we set  $ \beta_k = c \cdot (d\sqrt{\logt} + \zeta\sqrt{kd})H  $ for some absolute constant $c>0$, we have 
 \begin{equation*}
 \sum_{k=1}^K \beta_k^2  =  c^2 \sum_{k=1}^K (d\sqrt{\logt} +  \zeta\sqrt{kd}) ^2  H^2  \leq 2 c^2 \sum_{k=1}^K  ( d^2\logt + \zeta^2 k d) H^2  \le 2c^2(d^2HT\logt + \zeta^2T^2d),
 \end{equation*}
 which implies that 
 \begin{equation}
 \label{eq:mis5}
 \bigg( \sum_{k=1}^K \beta_k^2 \biggr )^{1/2} \leq \sqrt{2} c\bigl ( d \sqrt{H T \logt} + \zeta  T \sqrt{d} \bigr).
 \end{equation}
 Therefore, combining Eq.~\eqref{eq:mis3}, Eq.~\eqref{eq:mis4}, and Eq.~\eqref{eq:mis5}, we have 
 \#\label{eq:mis6}
 \sum_{k=1}^{K}\beta_k\sum_{h=1}^H \sqrt{(\bphi^{k}_h)\trans (\Lambda^k_h)^{-1}  \bphi^{k}_h}
   \le 2c\cdot (\sqrt{d^3 H^3 T\logt^2}  + \zeta d  H T\sqrt{\logt} ).
 \#
  Finally, combining Eq.~\eqref{eq:mis1}, Eq.~\eqref{eq:mis2}, and  Eq.~\eqref{eq:mis6}, we obtain 
  \$
  \text{Regret}(K)  \le c''\cdot (\sqrt{d^3 H^3 T\logt^2}  + \zeta d  H T\sqrt{\logt} ),
  \$
 for some absolute constant $c''$. This concludes the proof of the theorem.
 \end{proof}



\section{Auxiliary Lemmas} \label{sec:aux_lem}

This section presents several auxiliary lemmas and their proofs.

\subsection{Important inequalities for summations}
First, we present a few important short inequalities for summations.

\begin{lemma}\label{lem:basic_ineq} Let $\Lambda_t = \lambda \I + \sum_{i=1}^t \bphi_i \bphi_i\trans$ where $ \bphi_i \in \R^d$ and $\lambda > 0$. Then:
\begin{equation*}
\sum_{i=1}^t \bphi_i\trans (\Lambda_t)^{-1} \bphi_i \le d.
\end{equation*}
\end{lemma}
\begin{proof}
We have $\sum_{i=1}^t \bphi_i\trans (\Lambda_t)^{-1} \bphi_i
= \sum_{i=1}^t\tr( \bphi_i\trans (\Lambda_t)^{-1} \bphi_i)
= \tr((\Lambda_t)^{-1} \sum_{i=1}^t\bphi_i\bphi_i\trans)$.
Given the eigenvalue decomposition  $\sum_{i=1}^t\bphi_i\bphi_i\trans = \U \mathrm{diag}(\lambda_1,  \ldots, \lambda_d)  \U\trans$, we have
$\Lambda_t = \U \mathrm{diag}(\lambda_1+\lambda,  \ldots, \lambda_d+\lambda) \U\trans$, and $\tr((\Lambda_t)^{-1} \sum_{i=1}^t\bphi_i\bphi_i\trans)=
\sum_{j=1}^d \lambda_j/(\lambda_j + \lambda) \le d$.
\end{proof}

\begin{lemma} [\cite{abbasi2011improved}]\label{lemma:telescope}
	Let $\{\bphi_t \}_{t\geq 0}$ be a bounded sequence in $\R^d$  satisfying $\sup_{t\geq 0}\| \bphi_t \| \leq 1$. Let $\Lambda_0  \in \R^{d\times d}$ be a positive definite matrix.  For any $t\geq 0$, we define $
	\Lambda_t = \Lambda_0 + \sum_{j = 1}^t \bphi_ j ^\top   \bphi_j$.  Then, if the smallest eigenvalue of $\Lambda_0$  satisfies $\lambda_{\min}(\Lambda_0) \geq 1 $, we have 
	$$
	\log \biggl [  \frac{\det(
		\Lambda_t )}{\det(\Lambda_0)}\biggr ] \leq \sum_{j=1}^{t} \bphi_j^\top \Lambda_{j-1} ^{-1} \bphi_j \leq 2 \log \biggl [  \frac{\det(
		\Lambda_t )}{\det(\Lambda_0)}\biggr ].
	$$ 
	\end{lemma}
\begin{proof}
	
	Since $\lambda_{\min} (\Lambda_0) \geq 1$ and  $ \| \bphi_t  \| \leq 1$ for all $j\ge 0$, we have 
$$
 \bphi_j ^\top  \Lambda_{j-1}^{-1}  \bphi_j   \leq [\lambda_{\min} (\Lambda_0) ]^{-1}  \cdot  \| \bphi_j  \| ^2 \leq 1 , \qquad  \forall j\geq 0. 
$$
Note that, for any $x \in [0,1]$, it holds that $\log (1+x) \le x \le 2 \log (1+x)$.  Therefore, we have 
\#\label{eq:final12}
 \sum_{j=1}^{t} \log  \bigl (  1 + \bphi_j ^\top  \Lambda_{j-1}^{-1}   \bphi_j  \bigr ) \le \sum_{j=1}^{t} \bphi_j ^\top  \Lambda_{j-1} ^{-1}  \bphi_j    \leq  2\sum_{j=1}^{t}  \log  \bigl (  1 + \bphi_j ^\top  \Lambda_{j-1}^{-1}   \bphi_j  \bigr ). 
\# 
Moreover, for any $t\geq 0 $, by the definition of $\Lambda_t$, we have 
$$
\det (\Lambda_t ) = \det(  \Lambda_{t-1}   +  \bphi_{t}  \bphi_t \trans  )  =  \det(  \Lambda_{t-1}  ) \cdot \det ( \I +  \Lambda_{t-1}^{-1/2}\bphi_{t}  \bphi_t \trans\Lambda_{t-1}^{-1/2} )  . 
$$
Since $\det ( \I +  \Lambda_{t-1}^{-1/2}\bphi_{t}  \bphi_t \trans\Lambda_{t-1}^{-1/2} )
= 1+ \bphi_t \trans\Lambda_{t-1}^{-1}\bphi_{t}$, the recursion  gives:
\begin{equation}\label{eq:final4}
\sum_{j=1}^{t}  \log  \bigl (  1 + \bphi_j ^\top  \Lambda_{j-1}^{-1}   \bphi_j  \bigr )
= \logdet (\Lambda_t  ) - \logdet ( \Lambda _0 ). 
\end{equation}
Therefore, combining Eq.~\eqref{eq:final12} and Eq.~\eqref{eq:final4}, we conclude the proof.
\end{proof}

	

\subsection{Concentration inequalities for self-normalized processes}

Next, we present a few concentration inequalities. The following one provides a concentration inequality for the standard self-normalized processes.
\begin{theorem}[Concentration of Self-Normalized Processes \cite{abbasi2011improved}] \label{thm:self_norm}
Let $\{\epsilon_t\}_{t = 1}^\infty$ be a real-valued stochastic process with corresponding filtration $\{\cF_t\}_{t = 0}^\infty$. Let $\epsilon_t | \cF_{t-1}$ be zero-mean and $\sigma$-subGaussian; i.e. $\E[\epsilon_t | \cF_{t-1}] = 0$, and
\begin{equation*}
\forall \lambda \in \R, \qquad \E[e^{\lambda \epsilon_t} |\cF_{t-1}] \le e^{\lambda^2 \sigma^2/2}. 
\end{equation*} 
Let $\{\bphi_t\}_{t = 0}^\infty$ be an $\R^d$-valued stochastic process where $\bphi_t \in \cF_{t-1}$. Assume $\Lambda_0$ is a $d\times d$ positive definite matrix, and let $\Lambda_t = \Lambda_0 + \sum_{s=1}^t \bphi_s \bphi_s\trans$. Then for any $\delta>0$, with probability at least $1-\delta$, we have for all $t \ge 0$:
\begin{equation*}
\norm{\sum_{s = 1}^t \bphi_s \epsilon_s }^2_{\Lambda_t^{-1}}
\le 2\sigma^2 \log \left[ \frac{\det(\Lambda_t)^{1/2}\det(\Lambda_0)^{-1/2}}{\delta} \right]. 
\end{equation*}
\end{theorem}


When specializing this concentration inequality to our setting, we require uniform concentration over all value functions $V$ within a function class $\mathcal{V}$. This uniform concentration incurs an additional term that depends logarithmically on the covering number of $\mathcal{V}$.
\begin{lemma}\label{lem:self_norm_covering}
Let $\{x_\tau\}_{\tau = 1}^\infty$ be a stochastic process on state space $\cS$ with corresponding filtration $\{\cF_\tau\}_{\tau = 0}^\infty$. Let $\{\bphi_\tau\}_{\tau = 0}^\infty$ be an $\R^d$-valued stochastic process where $\bphi_\tau \in \cF_{\tau-1}$, and $\norm{\bphi_\tau} \le 1$. Let $\Lambda_k = \lambda I + \sum_{\tau=1}^k \bphi_\tau \bphi_\tau\trans$. Then for any $\delta >0$, with probability at least $1-\delta$, for all $k \ge 0$, and any $V \in \mathcal{V}$ so that $\sup_x |V(x)| \le H$, we have:
\begin{equation*} 
\norm{\sum_{\tau = 1}^k \bphi_\tau  \bigl \{ V(x_\tau) - \E[V(x_\tau)|\cF_{\tau-1}] \bigr \} }^2_{\Lambda_k^{-1}}
\le 4H^2 \left[ \frac{d}{2}\log\biggl( \frac{k+\lambda}{\lambda}\biggr )  + \log\frac{\mathcal{N}_{\epsilon}}{\delta}\right]  + \frac{8k^2\epsilon^2}{\lambda}, 
\end{equation*}
where $\mathcal{N}_{\epsilon}$ is the $\epsilon$-covering number of $\mathcal{V}$ with respect to the distance $\mathrm{dist}(V, V') = \sup_{x} |V(x) - V'(x)|$.
\end{lemma}

\begin{proof}
For any $V \in \mathcal{V}$, we know there exists a $\tilde{V}$ in the $\epsilon$-covering such that
\begin{equation*}
V = \tilde{V} + \Delta_V \quad\text{~and~}\quad \sup_x |\Delta_V(x)| \le \epsilon. 
\end{equation*}
This gives following decomposition:
\begin{align*}
&\norm{\sum_{\tau = 1}^k \bphi_\tau \bigl \{ V(x_\tau) - \E[V(x_\tau)|\cF_{\tau-1}]\bigr \}  }^2_{\Lambda_k^{-1}}\\
&\qquad \le  2\norm{\sum_{\tau = 1}^k \bphi_\tau \big\{ \tilde{V}(x_\tau) - \E[\tilde{V}(x_\tau)|\cF_{\tau-1}] \bigr \} }^2_{\Lambda_k^{-1}}
+ 2\norm{\sum_{\tau = 1}^k \bphi_\tau \bigl \{ \Delta_V(x_\tau) - \E[\Delta_V(x_\tau)|\cF_{\tau-1}]\bigr \} }^2_{\Lambda_k^{-1}},
\end{align*}
where we can apply Theorem \ref{thm:self_norm} and a union bound to the first term. Also, it is not hard to bound the second term by $8k^2\epsilon^2/\lambda$.
\end{proof}

To compute the covering number of function class $\mathcal{V}$, we first require a basic result on the covering number of a Euclidean ball as follows. We refer readers to classical material, such as Lemma 5.2 in \cite{vershynin2010introduction}, for its proof.

\begin{lemma} [Covering Number of Euclidean Ball] \label{lem:euclid}
    For any $\epsilon > 0$, the $\epsilon$-covering number of the  Euclidean ball in $\R^d$ with radius $R> 0$ is upper bounded by $(1 + 2 R/ \epsilon )^d$. 
\end{lemma}

Now, we are ready to compute the covering number of $\mathcal{V}$.

\begin{lemma} \label{lem:covering_number}
Let $\mathcal{V}$ denote a class of functions mapping from $\cS$ to $\R$ with following parametric form
$$V(\cdot) = \min \Bigl \{\max_a \w\trans\bphi(\cdot, a) + \beta \sqrt{\bphi(\cdot, a) \trans\Lambda^{-1} \bphi(\cdot, a)}, H \Bigr \},$$
where the parameters $(\w, \beta, \Lambda)$ satisfy $\norm{\w} \le L$, $\beta \in [0, B] $ and the minimum eigenvalue satisfies $\lambda_{\min}(\Lambda) \ge \lambda$. Assume $\norm{\bphi(x, a)}\le 1$ for all $(x, a)$ pairs, and let $\mathcal{N}_{\epsilon}$ be the $\epsilon$-covering number of $\mathcal{V}$ with respect to the distance $\mathrm{dist}(V, V') = \sup_{x} |V(x) - V'(x)|$. Then
\begin{equation*}
\log \mathcal{N}_{\epsilon} \le d  \log (1+ 4L / \epsilon ) + d^2 \log \bigl [ 1 +  8 d^{1/2} B^2  / (\lambda\epsilon^2)  \bigr ].
\end{equation*}
\end{lemma}


\begin{proof}
Equivalently, we can reparametrize the function class $\mathcal{V}$ by let $\A = \beta^2 \Lambda^{-1}$, so we have
\begin{align}\label{eq:reparamV}
V (\cdot) = \min \Bigl \{\max_a \w\trans\bphi(\cdot, a) + \sqrt{\bphi(\cdot, a) \trans \A \bphi(\cdot, a)}, H\Bigr \}
\end{align}
for $\norm{\w} \le L$ and $\norm{\A} \le B^2 \lambda^{-1}$.  
For any two functions $V_1, V_2 \in \mathcal{V}$, let them take the form in Eq.~\eqref{eq:reparamV} with parameters $(\w_1, \A_1)$ and $(\w_2, \A_2)$, respectively. 
Then, since both $\min\{\cdot, H\}$ and $\max_a$ are contraction maps, 
we have 
\begin{align} \label{eq:cover_upper}
\mathrm{dist}(V_1, V_2)  & \le \sup_{x, a} ~\Bigl | \Bigl[  \w_1\trans\bphi(x, a) + \sqrt{\bphi(x, a) \trans \A_2 \bphi(x, a)  } \Bigr ] -  \Bigl[ \w_2\trans\bphi(x, a) + \sqrt{\bphi(x, a) \trans \A_2 \bphi(x, a)  } \Bigr ] \Bigr |  \notag \\
&\le  \sup_{\bphi:\norm{\bphi}\le 1} ~\Bigl | \Bigl[  \w_1\trans\bphi + \sqrt{\bphi \trans \A_2 \bphi  } \Bigr ] -  \Bigl[ \w_2\trans\bphi +  \sqrt{\bphi \trans \A_2 \bphi  } \Bigr ] \Bigr |  \notag \\
& \leq \sup_{\bphi:\norm{\bphi}\le 1}   \bigl | (\w_1 - \w_2) \trans \bphi \bigr | +  \sup_{\bphi:\norm{\bphi}\le 1} \sqrt{  \bigl|  \bphi\trans (\A_1 -\A_2) \bphi    \bigr |  } \notag \\
& = \| \w_1 - \w_2 \| + \sqrt{\| \A_1 - \A_2 \|} \le \| \w_1 - \w_2 \| + \sqrt{\| \A_1 - \A_2 \|_F}, 
\end{align}
where the second last inequality follows from the fact that $| \sqrt{ x} - \sqrt{y} |  \leq \sqrt{ | x- y|}$ holds for any $x, y \geq 0$.
For matrices, $\norm{\cdot}$ and $\fnorm{\cdot}$ denote the matrix operator norm and Frobenius norm respectively.

Let $\mathcal{C}_{\w}$ be an $\epsilon/2$-cover of $\{\w\in\R^d | \norm{\w} \le L\}$ with respect to the 2-norm, and $\mathcal{C}_{\A}$ be an $\epsilon^2/4$-cover of $\{\A \in \R^{d\times d} | \norm{\A}_F \le d^{1/2}B^2\lambda^{-1}\}$ with respect to the Frobenius norm. By Lemma \ref{lem:euclid}, we know:
\begin{equation*}
| \mathcal{C}_\w | \leq ( 1+ 4L / \epsilon ) ^d , \qquad  | \mathcal{C}_\A | \leq \bigl [ 1 +  8 d^{1/2} B^2 / (\lambda\epsilon^2)  \bigr ] ^{d^2}.
\end{equation*}
By Eq.~\eqref{eq:cover_upper}, for any $V_1 \in \mathcal{V}$, there exists $\w_2 \in \mathcal{C}_\w$ and $\A_2 \in  \mathcal{C}_\A$ such that $V_2$ parametrized by $(\w_2, \A_2)$ satisfies $\mathrm{dist}(V_1, V_2)  \leq \epsilon$. Hence, it holds that $\mathcal{N}_{\epsilon} \leq | \mathcal{C}_\w | \cdot | \mathcal{C}_\A |$, which gives:
$$
\log \mathcal{N}_{\epsilon}  \leq \log  | \mathcal{C}_\w |  +  \log | \mathcal{C}_\A |  \leq d  \log (1+ 4L / \epsilon ) + d^2 \log \bigl [  1 +  8 d^{1/2} B^2  / (\lambda\epsilon^2)  \bigr ].
$$
This concludes the proof.
\end{proof}

\end{document}